\pdfoutput=1

\documentclass[11pt,a4paper]{article}

\usepackage{graphicx}  
\usepackage{times}
\usepackage{latexsym}

\usepackage{url}
\usepackage{amsmath}
\usepackage{amsthm}
\usepackage{amssymb}
\usepackage{bbm}
\usepackage{mathtools}
\usepackage{paralist}

\usepackage{pgf}
\usepackage{tikz}
\usetikzlibrary{arrows,automata}

\theoremstyle{definition}
\newtheorem{definition}{Definition}
\theoremstyle{plain}
\newtheorem{theorem}{Theorem}
\newtheorem{lemma}{Lemma}
\newtheorem{corollary}{Corollary}

\usepackage{xpatch}
\makeatletter
\AtBeginDocument{\xpatchcmd{\@thm}{\thm@headpunct{.}}{\thm@headpunct{}}{}{}}
\makeatother

\newcommand{\bv}[1]{\mathbf{#1}}
\newcommand{\bm}[1]{\mathbf{#1}}

\DeclarePairedDelimiter\abs{\lvert}{\rvert}%
\DeclarePairedDelimiter\norm{\lVert}{\rVert}%

\makeatletter
\let\oldabs\abs
\def\abs{\@ifstar{\oldabs}{\oldabs*}}
\let\oldnorm\norm
\def\norm{\@ifstar{\oldnorm}{\oldnorm*}}
\makeatother

\newcommand{\anbn}{a^nb^n}

\newcommand{\contanbn}{a^nb^n\Sigma^*}
\newcommand{\MLPT}[1]{D_{#1}}
\newcommand{\decfun}[1]{\bv d_{#1}}

\newcommand{\concat}{{\cdot}}

\newcommand{\zcheck}{\vec{\mathbbm{1}}_{=0}}

\newcommand{\num}[2]{\#_{#1}(#2)}

\newcommand{\computes}[1]{\rightarrow_{#1}}

\newcommand{\citeinp}[1]{\citeauthor{#1}, \citeyear{#1}}

\usepackage{xspace}
\newcommand{\model}{encoder\xspace}
\newcommand{\modelpl}{encoders\xspace}
\newcommand{\Modelpl}{Encoders\xspace}

\newcommand{\mdet}{an\xspace}
\newcommand{\mDet}{An\xspace}

\newcommand{\capacity}{state expressiveness\xspace}

\DeclareMathOperator{\relu}{\mathrm{ReLU}}

\newcommand{\fixfun}{\mathrm{fix}}

\newcommand{\rat}{\mathcal{R}}

\newcommand{\clean}[1]{\textrm{\normalfont #1}}

\usepackage{svg}

\makeatletter
\define@key{svg.sty}{scale}{}
\makeatother

\usepackage[hyperref]{acl2020}
\usepackage{times}
\usepackage{latexsym}

\usepackage{microtype}

\aclfinalcopy 
\def\aclpaperid{1032} 

\title{A Formal Hierarchy of RNN Architectures}

\author{
    William Merrill\footnotemark[1]\: \hspace{2.0em}
    Gail Weiss\footnotemark[2]\: \hspace{2.0em}
    Yoav Goldberg\footnotemark[1]\: \footnotemark[3]\: \\
    \bf{Roy Schwartz}\footnotemark[1]\: \footnotemark[4]\: \hspace{2.0em}
    \bf{Noah A.~Smith}\footnotemark[1]\: \footnotemark[4]\: \hspace{2.0em}
    \bf{Eran Yahav}\footnotemark[2]\:\\
    \footnotemark[1]\: Allen Institute for AI  \hspace{.2em}
    \footnotemark[2]\: Technion \hspace{.2em}
    \footnotemark[3]\: Bar Ilan University \hspace{.2em}
    \footnotemark[4]\: University of Washington \\
    \texttt{\{willm,yoavg,roys,noah\}@allenai.org}\\
    \texttt{\{sgailw,yahave\}@cs.technion.ac.il}
 }

\date{\today}

\begin{document}
\maketitle



\begin{abstract}
We develop a formal hierarchy of the expressive capacity of RNN architectures. The hierarchy is based on two formal properties: space complexity, which measures the RNN's memory, and rational recurrence, defined as whether the recurrent update can be described by a weighted finite-state machine. We place several RNN variants within this hierarchy. For example, we prove the LSTM is not rational, which formally separates it from the related QRNN \citep{bradbury2016qrnn}.
We also show how these models' expressive capacity is expanded 
by stacking multiple layers or composing them with different pooling functions.
Our results build on the theory of ``saturated" RNNs \citep{merrill-2019-sequential}. While formally extending these findings to unsaturated RNNs is left to future work, we hypothesize that the practical learnable capacity of unsaturated RNNs obeys a similar hierarchy. 
Experimental findings from training unsaturated networks on formal languages support this conjecture.
\textbf{We report updated experiments in \autoref{sec:erratum}.}
\end{abstract}
\section{Introduction} \label{Sec:Intro}

\begin{figure}
    \centering
    \def\svgwidth{208pt}
    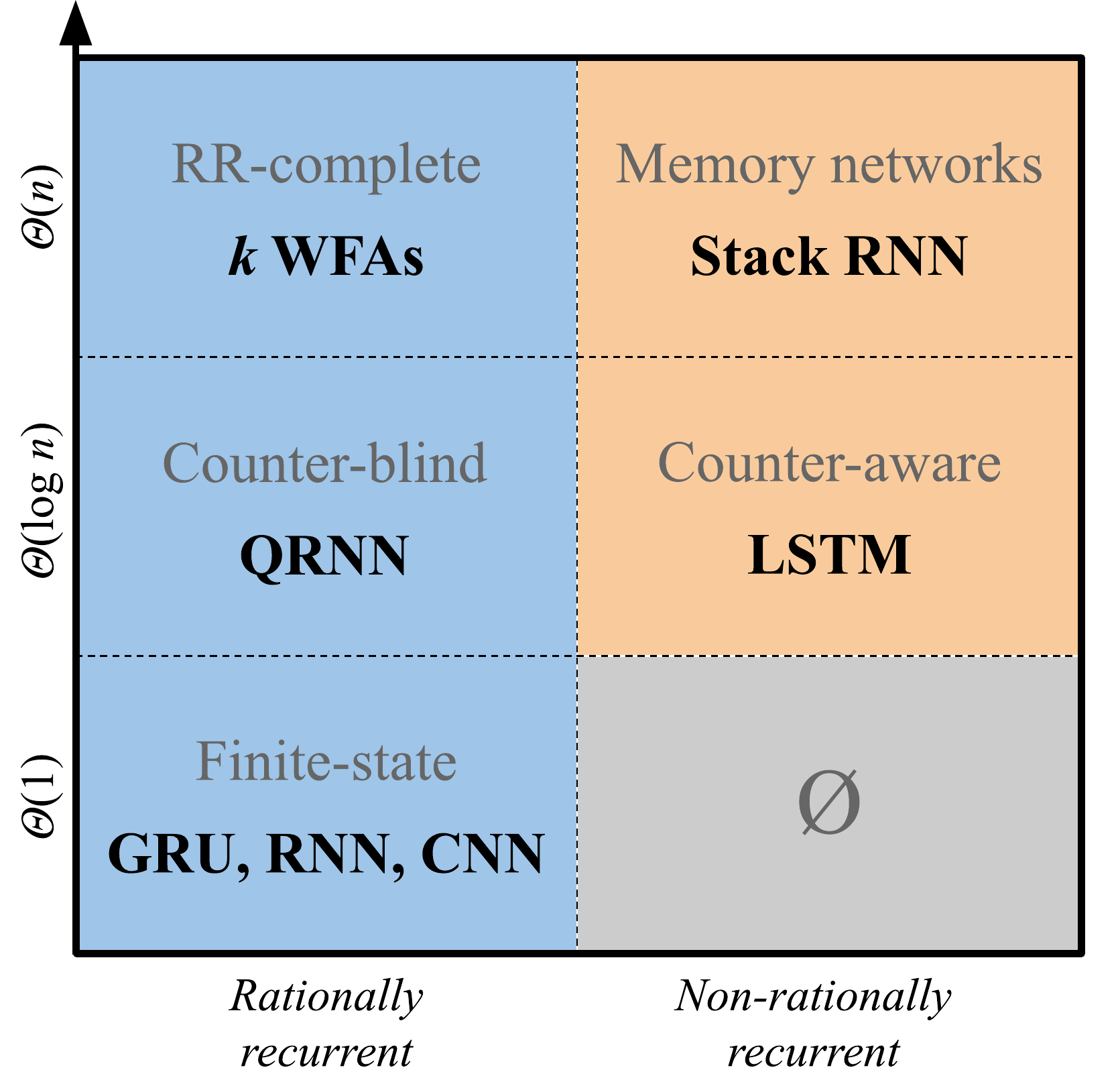
    \caption{Hierarchy of state expressiveness for saturated RNNs and related models. 
    The $y$ axis represents 
    increasing 
    space complexity. 
    $\emptyset$ means provably empty.
    Models are in bold
    with qualitative descriptions in gray. 
    \label{fig:hierarchy1}}
\end{figure}

While neural networks are central to the performance of today's strongest NLP systems, theoretical understanding of the formal properties of different kinds of networks is still limited.  It is established, for example, that the \citet{elman1990finding} RNN is Turing-complete, given infinite precision and computation time \cite{siegelmann1992turing,siegelmann1994analog,chen2017recurrent}. But tightening these unrealistic assumptions has serious implications for expressive power \citep{weiss2018}, leaving a significant gap between classical theory and practice, which theorems in this paper attempt to address.

Recently, \citet{peng2018rational} introduced \textbf{rational RNNs}, a subclass of RNNs whose internal state can be computed by independent weighted finite automata (WFAs). Intuitively, such models have a computationally simpler recurrent update than conventional models like long short-term memory networks (LSTMs; \citealp{hochreiter1997long}).
Empirically, rational RNNs like the quasi-recurrent neural network (QRNN; \citealp{bradbury2016qrnn}) and unigram rational RNN \citep{dodge-etal-2019-rnn} perform comparably to the LSTM, with a smaller computational budget. Still, the underlying simplicity of rational models raises the question of whether their expressive power is fundamentally limited compared to other RNNs.

In a separate line of work, \citet{merrill-2019-sequential} introduced the \textbf{saturated RNN}\footnote{Originally referred to as the \emph{asymptotic RNN}.} as a formal model for analyzing the capacity of RNNs. A saturated RNN is a simplified network where all activation functions have been replaced by step functions. The saturated network may be seen intuitively as a ``stable'' version of its original RNN, in which the internal activations act discretely. A growing body of work---including this paper---finds that the saturated theory predicts differences in practical learnable capacity for various RNN architectures \citep{weiss2018, merrill-2019-sequential, suzgun-etal-2019-lstm}.

We compare the expressive power of rational and non-rational RNNs, distinguishing between \textit{state expressiveness} (what kind and amount of information the RNN states can capture) and \textit{language expressiveness} (what languages can be recognized when the state is passed to a classifier). To do this, we build on the theory of saturated RNNs.

\paragraph{State expressiveness} We introduce a unified hierarchy (\autoref{fig:hierarchy1}) of the functions expressible by the states of rational and non-rational RNN encoders. The hierarchy is defined by two formal properties: space complexity, which is a measure of network memory,\footnote{Space complexity measures the number of different configurations an RNN can reach as a function of input length. Formal definition deferred until \autoref{sec:definitions}.} and rational recurrence, whether the internal structure of the RNN can be described by WFAs. 
The hierarchy reveals concrete differences between LSTMs and QRNNs, and further separates both from a class containing convolutional neural networks (CNNs, \citealp{lecun1998convolutional, kim2014convolutional}), Elman RNNs, and gated recurrent units (GRU; \citealp{cho2014learning}). 

We provide the first formal proof that LSTMs can encode functions that rational recurrences cannot.
On the other hand, we show that the saturated Elman RNN and GRU are rational recurrences with constant space complexity, whereas the QRNN has unbounded space complexity. We also show that an unrestricted WFA has rich expressive power beyond any saturated RNN we consider---including the LSTM. This difference potentially opens the door to more expressive RNNs incorporating the computational efficiency of rational recurrences.

\paragraph{Language expressiveness}
When applied to classification tasks like language recognition, RNNs are typically 
combined with a ``decoder'': additional layer(s) that map their hidden states to a prediction.
Thus, despite differences in state expressiveness, rational RNNs might be able to achieve comparable empirical performance to non-rational RNNs 
on NLP tasks. 
In this work, we consider the setup in which the decoders only view the final hidden state of the RNN.\footnote{This is common, but not the only possibility. For example, an attention decoder observes the full sequence of states.}
We demonstrate that
a sufficiently strong decoder
can overcome some of the differences in state expressiveness between different models. 
For example, an LSTM can recognize $a^nb^n$ with a single decoding layer, whereas a QRNN provably cannot until the decoder has two layers. 
However, we also construct a language that an LSTM can recognize without a decoder, but a QRNN cannot recognize with any decoder.
Thus, no decoder can fully compensate for the weakness of the QRNN compared to the LSTM. 


\paragraph{Experiments} Finally, we conduct experiments on formal languages,
justifying that our theorems
correctly predict which languages unsaturated recognizers trained by gradient descent can learn.
Thus, we view our hierarchy as a useful formal tool for understanding the relative capabilities of different RNN architectures.

\paragraph{Roadmap} We present the formal devices for our analysis of RNNs in \autoref{sec:definitions}. In \autoref{Sec:Hierarchy} we develop our hierarchy of state expressiveness for single-layer RNNs. In \autoref{sec:languages}, we shift to study RNNs as language recognizers. Finally, in \autoref{Sec:Experiments}, we provide empirical results evaluating the
relevance of our predictions
for unsaturated RNNs.

\section{Building Blocks} \label{sec:definitions}

In this work, we analyze RNNs using formal models from automata theory---in particular, WFAs and counter automata. In this section, we first define the basic notion of an encoder studied in this paper, and then introduce more specialized formal concepts: WFAs, counter machines (CMs), space complexity, and, finally, various RNN architectures.

\subsection{Encoders}

We view both RNNs and automata as \textit{\modelpl}: machines that can be parameterized to compute a set of functions $f : \Sigma^* \rightarrow \mathbb{Q}^k$, where $\Sigma$ is an input alphabet and $\mathbb{Q}$ is the set of rational reals. Given an encoder $M$ and parameters $\theta$, we use $M_\theta$ to represent the specific function that the parameterized encoder computes. For each \model, we refer to the set of functions that it can compute as its \textit{\capacity}. For example, a deterministic finite state acceptor (DFA) is \mdet \model whose parameters are its transition graph. Its \capacity is the indicator functions for the regular languages. 



\subsection{WFAs}

Formally, a WFA is a non-deterministic finite automaton where each starting state, transition, and final state is weighted. Let $Q$ denote the set of states, $\Sigma$ the alphabet, and $\mathbb{Q}$ the rational reals.\footnote{WFAs are often defined over a generic semiring; we consider only the special case when it is the field of rational reals.}  This weighting is specified by three functions:

\begin{compactenum}
    \item Initial state weights $\lambda : Q \rightarrow \mathbb{Q}$
    \item Transition weights $\tau : Q \times \Sigma \times Q \rightarrow \mathbb{Q}$
    \item Final state weights $\rho : Q \rightarrow \mathbb{Q}$    
\end{compactenum}

\noindent The weights are used to encode any string $x \in \Sigma^*$:

\begin{definition}[Path score]
Let $\pi$ be a path of the form $q_0 \computes{x_1} q_1 \computes{x_2} \cdots \computes{x_t} q_t$ through WFA $A$. The score of $\pi$ is given by
\begin{equation*}
    A[\pi] = \lambda(q_0) \left( \textstyle\prod_{i=1}^t \tau(q_{i-1}, x_i, q_i) \right) \rho(q_t) .
\end{equation*}
\end{definition}

\noindent By $\Pi(x)$, denote the set of paths producing $x$.

\begin{definition}[String encoding] \label{def:string-score}
The encoding computed by a WFA $A$ on string $x$ is 
\begin{equation*}
    A[x] = \textstyle \sum_{\pi \in \Pi(x)} A[\pi] .
\end{equation*}
\end{definition}

\paragraph{Hankel matrix} Given a function $f : \Sigma^* \rightarrow \mathbb{Q}$ and two enumerations $\alpha, \omega$ of the strings in $\Sigma^*$, we define the Hankel matrix of $f$ as the infinite matrix
\begin{equation}
    [H_f]_{ij} = f(\alpha_i\concat \omega_j) .
\end{equation}
where $\concat$ denotes concatenation.
It is sometimes convenient to treat $H_f$ as though it is directly indexed by $\Sigma^*$, e.g. $[H_f]_{\alpha_i,\omega_j}=f(\alpha_i\concat\omega_j)$, or refer to a sub-block of a Hankel matrix, row- and column- indexed by prefixes and suffixes $P,S\subseteq\Sigma^*$.
The following result relates the Hankel matrix to WFAs:

\begin{theorem}[\citeinp{carlyle1971}; \citeinp{fliess1974matrices}] \label{thm:rank}
    For any $f: \Sigma^* \rightarrow \mathbb{Q}$, there exists a WFA that computes $f$ if and only if $H_f$ has finite rank.
\end{theorem}

\paragraph{Rational series {\normalfont \citep{sakarovitch2009rational}}} For all $k \in \mathbb{N}$, $ \bv f : \Sigma^* \rightarrow \mathbb{Q}^k$ is a \textit{rational series} if there exist WFAs $A_1, \cdots, A_k$ such that, for all $x \in \Sigma^*$ and $1 \leq i \leq k$, $A_i[x] = f_i(x)$.

\subsection{Counter Machines} \label{subsec:cms}

We now turn to introducing a different type of \model: the real-time counter machine (CM; \citealp{anonymous2020counters, fischer1966turing, fischer1968counter}). CMs are deterministic finite-state machines augmented with finitely many integer counters. While processing a string, 
the machine updates these counters,
and may use them to inform its behavior.

We view counter machines as \modelpl mapping $\Sigma^* \rightarrow \mathbb{Z}^k$. For $m \in \mathbb{N},\circ\in\{+,-,\times\}$, let $\circ m$ denote the function $f(n) = n \circ m$. 

\begin{definition}[General CM; \citeauthor{anonymous2020counters}, \citeyear{anonymous2020counters}]
A $k$-counter CM is a tuple $\langle \Sigma, Q, q_0, u, \delta \rangle$ with
\begin{compactenum}
    \item A finite alphabet $\Sigma$
    \item A finite set of states $Q$, with initial state $q_0$
    \item A counter update function 
    $$
        u : \Sigma \times Q \times \{0, 1\}^k \rightarrow 
        \{ \times 0, -1, +0, +1 \}^k
    $$
    \item A state transition function
    $$
        \delta : \Sigma \times Q \times \{0, 1\}^k \rightarrow Q
    $$
\end{compactenum}
\end{definition}

A CM processes input tokens $\{x_t\}_{t=1}^n$ sequentially. Denoting $\langle q_t , \bv c_t\rangle \in Q \times \mathbb{Z}^k$ a CM's configuration at time $t$, define its next configuration:
\begin{align}
q_{t+1} &= \delta \left( x_t, q_t, \zcheck \left(\bv c_t \right) \right)  \label{eq:delta} \\
\bv c_{t+1} &= u \left( x_t , q_t, \zcheck \left(\bv c_t \right) \right) (\bv c_t) , \label{eq:u}
\end{align}
where $\zcheck$ is a broadcasted ``zero-check" operation, i.e., $\zcheck(\bv v)_i \triangleq \mathbbm{1}_{=0}(v_i)$.
In \eqref{eq:delta} and \eqref{eq:u}, note that the machine only views the zeroness of each counter, and not its actual value. A general CM's encoding of a string $x$ is the value of its counter vector $\bv c_t$ after processing all of $x$.

\paragraph{Restricted CMs}
\begin{compactenum}
    \item A CM is \textit{$\Sigma$-restricted} iff $u$ and $\delta$ depend only on the current input $\sigma \in \Sigma$.
    \item A CM is \textit{$(\Sigma \times Q)$-restricted} iff $u$ and $\delta$ depend only on the current input $\sigma \in \Sigma$ and the current state $q \in Q$.
    \item A CM is \textit{$\Sigma^w$-restricted} iff it is $(\Sigma \times Q)$-restricted, and the states $Q$ are windows over the last $w$ input tokens, e.g., $Q=\Sigma^{\leq w}$.\footnote{The states $q\in\Sigma^{<w}$ represent the beginning of the sequence, before $w$ input tokens have been seen.}    
\end{compactenum}
These restrictions prevent the machine from being ``counter-aware": $u$ and $\delta$ cannot condition on the counters' values. As we will see, restricted CMs have natural parallels in the realm of rational RNNs. In \autoref{sec:rational-counting}, we consider the relationship between counter awareness and rational recurrence.

\subsection{Space Complexity}

As in \citet{merrill-2019-sequential}, we also analyze \modelpl in terms of state space complexity, measured in bits.

\begin{definition}[Bit complexity]
\mDet \model $M:\Sigma^* \rightarrow \mathbb{Q}^k$ has $T(n)$ space iff
\begin{equation*}
    \max_\theta \; \abs{\{s_{M_\theta}(x) \mid x \in\Sigma^{\leq n}\}}=2^{T(n)}, 
\end{equation*}
where $s_{M_\theta}(x)$ is a minimal representation\footnote{I.e., the minimal state representation needed to compute $M_\theta$ correctly. This distinction is important for architectures like attention, for which some implementations may retain unusable information such as input embedding order.}
of $M$'s internal configuration immediately after $x$.
\end{definition}

We consider three asymptotic space complexity classes: $\Theta(1)$, $\Theta(\log n)$, and $\Theta(n)$, corresponding to \modelpl that can reach a constant, polynomial, and exponential (in sequence length) number of configurations respectively. 
Intuitively, \modelpl that can dynamically count but cannot use more complex memory like stacks--such as all CMs--are in $\Theta(\log n)$ space. 
\Modelpl that can uniquely encode every input sequence are in $\Theta(n)$ space.

\subsection{Saturated Networks}
A saturated neural network is a discrete approximation of neural network considered by \citet{merrill-2019-sequential}, who calls it an ``asymptotic network.'' Given a parameterized neural encoder $M_\theta(x)$, we construct the saturated network $\clean{s-}M_\theta(x)$ by taking
\begin{equation}
    \clean{s-}M_\theta(x) = \lim_{N \rightarrow \infty} M_{N\theta}(x)
\end{equation}
\noindent where $N\theta$ denotes the parameters $\theta$ multiplied by a scalar $N$.
This transforms
each ``squashing'' function (sigmoid, tanh, etc.) to its extreme values (0, $\pm$1).
In line with prior work \citep{weiss2018, merrill-2019-sequential, suzgun2019memory}, we consider saturated networks a reasonable approximation for analyzing practical expressive power. For clarity, we denote the saturated approximation of an architecture by prepending it with s, e.g., s-LSTM.

\subsection{RNNs}

A recurrent neural network (RNN) is a parameterized update function $g_\theta:\mathbb{Q}^k \times \mathbb{Q}^{d_x} \rightarrow \mathbb{Q}^{k}$, where $\theta$ are the rational-valued parameters of the RNN and $d_x$ is the dimension of the input vector. $g_\theta$ takes as input a current state $\bv h\in\mathbb{Q}^{k}$ and input vector $\bv x\in\mathbb{Q}^{d_x}$, and produces the next state. Defining the initial state as $\bv h_0 = \bv 0$, an RNN can be applied to an input sequence $x \in(\mathbb{Q}^{d_x})^*$ one vector at a time to create a sequence of states $\{\bv h_t\}_{t\leq|x|}$, each representing an encoding of the prefix of $x$ up to that time step. RNNs can be used to encode sequences over a finite alphabet $x\in\Sigma^*$ by first applying a mapping (embedding) $e:\Sigma\rightarrow \mathbb{Q}^{d_x}$. 


\paragraph{Multi-layer RNNs} ``Deep'' RNNs are RNNs that have been arranged in $L$ stacked layers $R_1,...,R_L$. In this setting, the series of output states $\bv h_1, \bv h_2,..., \bv h_{|x|}$ generated by each RNN on its input is fed as input to the layer above it, and only the first layer receives the original input sequence $x\in\Sigma^*$ as input.

The recurrent update function $g$ can take several forms. The original and most simple form is that of the \emph{Elman RNN}. Since then, more elaborate forms using gating mechanisms have become popular, among them  the LSTM, GRU, and QRNN.

\paragraph{Elman RNNs {\normalfont \citep{elman1990finding}}} Let $\bv x_t$ be a vector embedding of $x_t$. For brevity, we suppress the bias terms in this (and the following) affine operations.
\begin{equation}
    \bv h_t = \tanh(\bv W \bv x_t + \bv U \bv h_{t-1}) .
\end{equation}
\noindent We refer to the saturated Elman RNN as the \textit{s-RNN}. The s-RNN has $\Theta(1)$ space \citep{merrill-2019-sequential}.

\paragraph{LSTMs {\normalfont \citep{hochreiter1997long}}}
An LSTM is a gated RNN with a state vector $\bv h_t \in \mathbb{Q}^{k}$ and memory vector $\bv c_t \in \mathbb{Q}^{k}$.
\footnote{
With respect to our presented definition of RNNs, the concatenation of $\bv h_t$ and $\bv c_t$ can be seen as the recurrently updated state. 
However in all discussions of LSTMs we treat only $\bv h_t$ as the LSTM's `state', in line with common practice.
}

\begin{align}
    \bv f_t &= \sigma(\bm W^f \bv x_t + \bm U^f \bv h_{t-1}) \\
    \bv i_t &= \sigma(\bm W^i \bv x_t + \bm U^i \bv h_{t-1}) \\
    \bv o_t &= \sigma(\bm W^o \bv x_t + \bm U^o \bv h_{t-1}) \\
    \bv{\tilde c_t} &= \tanh(\bm W^c \bv x_t + \bm U^c \bv h_{t-1}) \\
    \bv c_t &= \bv f_t \odot \bv c_{t-1} + \bv i_t \odot \bv{\tilde c_t} \\
    \bv h_t &= \bv o_t \odot \tanh(\bv c_t) .
\end{align}
\noindent The LSTM can use its memory vector $\bv c_t$ as a register of counters \cite{weiss2018}. \citet{merrill-2019-sequential} showed that the s-LSTM has $\Theta(\log n)$ space.

\paragraph{GRUs {\normalfont \citep{cho2014learning}}}
Another kind of gated RNN is the GRU.
\begin{align}
    \bv z_t &= \sigma(\bm W^z \bv x_t + \bm U^z \bv h_{t-1}) \\
    \bv r_t &= \sigma(\bm W^r \bv x_t + \bm U^r \bv h_{t-1}) \\
    \bv u_t &= \tanh \big( \bm W^u \bv x_t + \bm U^u(\bv r_t \odot \bv h_{t-1}) \big) \label{eq:gru_squashed} \\
    \bv h_t &= \bv z_t \odot \bv h_{t-1} + (1 - \bv z_t) \odot \bv u_t .
\end{align}
\noindent \citet{weiss2018} found that, unlike the LSTM, the GRU cannot use its memory to count dynamically. \citet{merrill-2019-sequential} showed the s-GRU has $\Theta(1)$ space.

\paragraph{QRNNs} \citet{bradbury2016qrnn} propose QRNNs as a computationally efficient hybrid of LSTMs and CNNs. Let $*$ denote convolution over time, 
let $\bv W^z, \bv W^f, \bv  W^o\in\mathbb{Q}^{d_x\times w \times k}$ be convolutions with window length $w$,
and let $\bv X\in\mathbb{Q}^{n\times d_x}$ denote the matrix of $n$ input vectors. An \textit{ifo}-QRNN (henceforth referred to as a QRNN) with \emph{window length} $w$ is defined by $\bv W^z,\bv W^f$, and $\bv W^o$ as follows:
\begin{align}
    \bv Z &= \tanh(\bv W^z * \bv X) \\
    \bv F &= \sigma(\bv W^f * \bv X) \\
    \bv O &= \sigma(\bv W^o * \bv X) \\
    \bv c_t &= \bv f_t \odot \bv c_{t-1} + \bv i_t \odot \bv z_t \\
    \bv h_t &= \bv o_t \odot \bv c_t
\end{align}
\noindent where $\bv z_t, \bv f_t, \bv o_t$ are respectively rows of $\bv Z, \bv F, \bv O$.
A QRNN $Q$ can be seen as an LSTM in which all uses of the state vector $\bv h_t$ have been replaced with a computation over the last $w$ input tokens--in this way it is similar to a CNN.

The s-QRNN has $\Theta(\log n)$ space, as the analysis of \citet{merrill-2019-sequential} for the s-LSTM directly applies. Indeed, any s-QRNN is also a ($\Sigma^w$)-restricted CM extended with ${=}{\pm}1$ (``set to ${\pm}1$'') operations.

\section{State Expressiveness}

\label{Sec:Hierarchy}
\begin{figure}
    \centering
    \def\svgwidth{\columnwidth}
    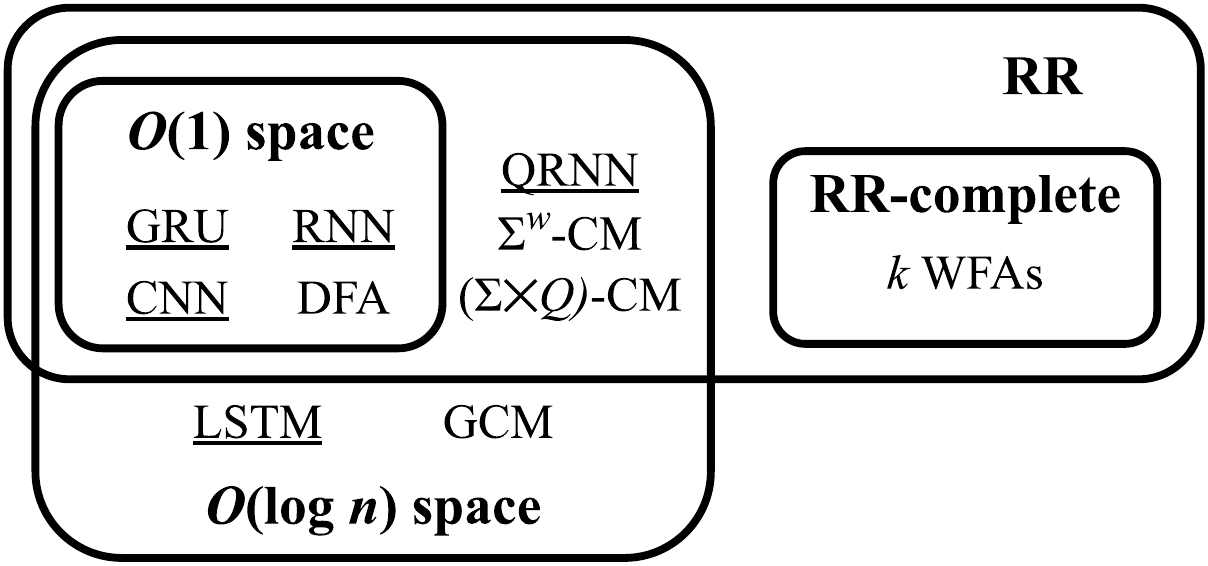
    \caption{Diagram of the relations between encoders. Neural networks are underlined. We group by asymptotic upper bound ($O$), as opposed to tight ($\Theta$). 
    \label{fig:hierarchy}}
\end{figure}

We now turn to presenting our results. 
In this section, we develop a hierarchy of single-layer RNNs based on their \capacity.
A set-theoretic view of the hierarchy is shown in  \autoref{fig:hierarchy}. 

Let $\rat$ be the set of rational series. The hierarchy relates $\Theta(\log n)$ space to the following sets:

\begin{itemize}
    \item \textbf{RR} As in \citet{peng2018rational}, we say that \mDet \model is \textit{rationally recurrent (RR)} iff its \capacity is a subset of $\rat$. 
    \item \textbf{RR-hard} \mDet \model is \textit{RR-hard} iff its \capacity contains $\rat$. A Turing machine is RR-hard, as it can simulate any WFA.
    \item \textbf{RR-complete} Finally, \mdet \model is \textit{RR-complete} iff its \capacity is equivalent to $\rat$. A trivial example of an RR-complete \model is a vector of $k$ WFAs.
\end{itemize}

The different RNNs are divided between the intersections of these classes. In \autoref{sec:beyond-rr}, we prove that the s-LSTM, already established to have $\Theta(\log n)$ space, is not RR. In \autoref{sec:rational-counting}, we demonstrate that \modelpl with restricted counting ability (e.g., QRNNs) are RR, and in \autoref{sec:finite-rr}, we show the same for all \modelpl with finite state (CNNs, s-RNNs, and s-GRUs). In \autoref{sec:rr-completeness}, we demonstrate that none of these RNNs are RR-hard. In \autoref{sec:self-attention}, we extend this analysis from RNNs to self attention.

\subsection{Counting Beyond RR} \label{sec:beyond-rr}
We find that \modelpl like the s-LSTM---which, as discussed in \autoref{subsec:cms}, is ``aware" of its current counter values---are not RR. To do this, we construct $f_0:\{a,b\}^* \rightarrow \mathbb{N}$ that requires counter awareness to compute on strings of the form $a^*b^*$, making it not rational. We then construct an s-LSTM computing $f_0$ over $a^*b^*$.

Let $\num{a-b}{x}$ denote the number of $a$s in string $x$ minus the number of $b$s.

\begin{definition}[Rectified counting]
    \begin{equation*}
    f_0 : x \mapsto
    \begin{cases}
    \num{a-b}{x} & \textrm{if} \; \num{a-b}{x} > 0 \\
    0 & \textrm{otherwise} .
    \end{cases}
\end{equation*}
\end{definition}

\begin{lemma} \label{thm:triangular-switch-counter} For all $f:\{a,b\}^* \rightarrow\mathbb{N}$, if $f(a^ib^j) = f_0(a^ib^j)$ for all $i,j\in\mathbb{N}$, then $f \not\in \rat$ .
\end{lemma}

\begin{proof}
Consider the Hankel sub-block $\bv A_n$ of $H_{f}$ with prefixes $P_n=\{a^i\}_{i\leq n}$ and suffixes $S_n=\{b^j\}_{j\leq n}$. $\bv A_n$ is lower-triangular:

\begin{equation}
    \begin{pmatrix}
    0 & 0 & 0 & \cdots \\
    1 & 0 & 0 & \cdots \\
    2 & 1 & 0 & \cdots \\
    \vdots & \vdots & \vdots &\ddots
    \end{pmatrix} .
\end{equation}

\noindent Therefore $\mathrm{rank}(\bv A_n)=n{-}1$.
Thus, for all $n$, there is a sub-block 
of $H_{f}$ with rank $n - 1$,
and so
$\mathrm{rank}(H_{f})$ is unbounded. It follows from \autoref{thm:rank} that there is no WFA computing $f$.
\end{proof}

\begin{theorem}
    The s-LSTM is not RR.
\end{theorem}

\begin{proof}
Assume the input has the form $a^ib^j$ for some $i,j$.
Consider the following LSTM
\footnote{In which $f_t$ and $o_t$ are set to $1$, such that $c_t = c_{t-1} + i_t \tilde c_t$.}:
\begin{align}
    i_t &= \sigma \big( 10N h_{t-1} - 2N \mathbbm{1}_{=b}(x_t) + N \big) \\
    \tilde c_t &= \tanh\big( N \mathbbm{1}_{=a}(x_t) - N \mathbbm{1}_{=b}(x_t) \big) \\
    c_t &= c_{t-1} + i_t \tilde c_t \\
    h_t &= \tanh(c_t) .
\end{align}

Let $N \rightarrow \infty$. Then $i_t=0$ iff $x_t=b$ and $h_{t-1} = 0$ (i.e. $c_{t-1} = 0$). Meanwhile, $\tilde c_t = 1$ iff $x_t = a$. The update term becomes
\begin{equation} \label{eq:lstm-update}
    i_t \tilde c_t = \begin{cases}
        1 & \textrm{if} \; x_t = a \\
        -1 & \textrm{if} \; x_t = b \; \textrm{and} \; c_{t-1} > 0 \\
        0 & \textrm{otherwise.}
    \end{cases}
\end{equation}

For a string $a^ib^j$, the update in \eqref{eq:lstm-update} is equivalent to the CM in \autoref{fig:slstm-cm}. Thus, by \autoref{thm:triangular-switch-counter}, the s-LSTM (and the general CM) is not RR.
\end{proof}

\begin{figure}
    \centering
    \begin{tikzpicture}[->,>=stealth',shorten >=1pt,auto,node distance=2.8cm, semithick]
    
    \node[initial,state]           (0) {$q_0$};
    
    \path (0) edge [loop above] node {$a/{+}1$} (0)
              edge [loop below] node {$b,{\neq}0/{-}1$} (0)
              edge [loop right]  node {$b,{=}0/{+}0$} (1);
    
    \end{tikzpicture}
    \caption{A $1$-CM computing $f_0$ for $x \in \{ a^ib^j \mid i, j \in \mathbb{N} \}$. Let $\sigma/{\pm}m$ denote a transition that consumes $\sigma$ and updates the counter by ${\pm}m$. We write $\sigma, {=}0 / {\pm}m$ (or ${\neq}$) for a transition that requires the counter is $0$.}
    \label{fig:slstm-cm}
\end{figure}
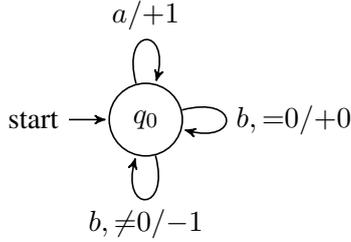

\subsection{Rational Counting} \label{sec:rational-counting}

While the counter awareness of a general CM enables it to compute non-rational functions, CMs that cannot view their counters are RR.

\begin{theorem} \label{thm:CM-level1}
Any $\Sigma$-restricted CM is RR.
\end{theorem}

\begin{proof}
We show that any function that a $\Sigma$-restricted CM can compute can also be computed by a collection of WFAs. The CM update operations (${-}1, {+}0, {+1},$ or ${\times}0$) can all be reexpressed in terms of functions $\bv r(x), \bv u(x) : \Sigma^* \rightarrow \mathbb{Z}^k$ to get:
\begin{align}
\bv c_t &= \bv r(x_t) \bv c_{t-1} + \bv u(x_t) \\
\bv c_t &= \textstyle \sum_{i=1}^t \left( \prod_{j=i + 1}^t \bv r(x_j) \right) \bv u(x_i) \label{eq:unrolled}.
\end{align}
\noindent A WFA computing $[\bv c_t]_i$ is shown in \autoref{fig:sigma-CM}.
\end{proof}

The WFA in \autoref{fig:sigma-CM} also underlies unigram rational RNNs \citep{peng2018rational}. Thus, $\Sigma$-restricted CMs are actually a special case of unigram WFAs. In \autoref{sec:restricted-counting}, we show the more general result:

\begin{theorem} \label{thm:CM-level2}
Any $(\Sigma \times Q)$-restricted CM is RR.
\end{theorem}

In many rational RNNs, the updates at different time steps are independent of each other outside of a window of $w$ tokens. \autoref{thm:CM-level2} tells us this independence is not an essential property of rational encoders. Rather, any CM where the update is conditioned by finite state (as opposed to being conditioned by a local window) is in fact RR.

Furthermore, since $(\Sigma^w)$-restricted CMs are a special case of $(\Sigma \times Q)$-restricted CMs, \autoref{thm:CM-level2} can be directly applied to show that the s-QRNN is RR. See \autoref{sec:restricted-counting} for further discussion of this.


\subsection{Finite-Space RR} \label{sec:finite-rr}

\autoref{thm:CM-level2} motivates us to also think about finite-space \modelpl: i.e., \modelpl with no counters" where the output at each prefix is fully determined by a finite amount of memory. The following lemma implies that any finite-space \model is RR:

\begin{lemma}
Any function $f : \Sigma^* \rightarrow \mathbb{Q}$ computable by a $\Theta(1)$-space \model is a rational series.
\end{lemma}

\begin{proof}
Since $f$ is computable in $\Theta(1)$ space, there exists a DFA $A_f$ whose accepting states are isomorphic to the range of $f$. We convert $A_f$ to a WFA by labelling each accepting state by the value of $f$ that it corresponds to. We set the starting weight of the initial state to $1$, and $0$ for every other state. We assign each transition weight $1$.
\end{proof}

\noindent Since the CNN, s-RNN, and s-GRU have finite state, we obtain the following result:

\begin{theorem} \label{thm:finite-state}
The CNN, s-RNN, and s-GRU are RR.
\end{theorem}

\noindent While \citet{schwartz2018sopa} and \citet{peng2018rational} showed the CNN to be RR over the max-plus semiring, \autoref{thm:finite-state} shows the same holds for $\langle \mathbb{Q}, \cdot, + \rangle$.

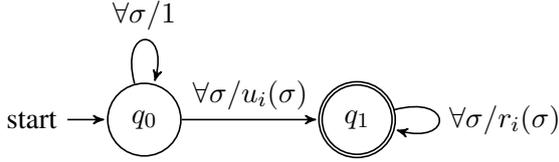
\begin{figure}
    \centering
    \begin{tikzpicture}[->,>=stealth',shorten >=1pt,auto,node distance=2.8cm, semithick]
    
    \node[initial,state]           (0) {$q_0$};
    \node[accepting,state]         (1) [right of=0] {$q_1$};
    
    \path (0) edge [loop above] node {$\forall \sigma/1$} (0)
              edge              node {$\forall \sigma/u_i(\sigma)$} (1)
          (1) edge [loop right] node {$\forall \sigma/r_i(\sigma)$} (1);
    
    \end{tikzpicture}
    \caption{WFA simulating unit $i$ of a $\Sigma$-restricted CM. Let $\forall \sigma/w(\sigma)$ denote a set of transitions consuming each token $\sigma$ with weight $w(\sigma)$. We use standard DFA notation to show initial weights $\lambda(q_0) = 1, \lambda (q_1) = 0$ and accepting weights $\rho(q_0) = 0, \rho(q_1) = 1$.}
    \label{fig:sigma-CM}
\end{figure}

\subsection{RR Completeness} \label{sec:rr-completeness}

While ``rational recurrence" is often used to indicate the simplicity of an RNN architecture, we find in this section that WFAs are surprisingly computationally powerful.
\autoref{fig:binary} shows a WFA mapping binary string to their numeric value, proving WFAs have $\Theta(n)$ space.
We now show that none of our RNNs are able to simulate an arbitrary WFA, even in the unsaturated form.

\begin{theorem}
Both the saturated and unsaturated RNN, GRU, QRNN, and LSTM\footnote{As well as CMs.} are not RR-hard.
\end{theorem}

\begin{proof}
Consider the function $f_b$ mapping binary strings to their value, e.g. $101 \mapsto 5$. The WFA in \autoref{fig:binary} shows that this function is rational.

The value of $f_b$ grows exponentially with the sequence length. On the other hand, the value of the RNN and GRU cell is bounded by $1$, and QRNN and LSTM cells can only grow linearly in time. 
Therefore, 
these \modelpl cannot compute $f_b$.
\end{proof}

In contrast, memory networks can have $\Theta(n)$ space. \autoref{sec:stack-rnns} explores this for stack RNNs.

\begin{figure}
    \centering
    \begin{tikzpicture}[->,>=stealth',shorten >=1pt,auto,node distance=2.8cm, semithick]
    
    \node[initial,state]           (0) {$q_0$};
    \node[accepting,state]         (1) [right of=0] {$q_1$};
    
    \path (0) edge [loop above] node {$\forall \sigma/1$} (0)
              edge              node {$\forall \sigma/\sigma$} (1)
          (1) edge [loop right] node {$\forall \sigma/2$} (1);
    
    \end{tikzpicture}
    \caption{A WFA mapping binary strings to their numeric value. This can be extended for any base $>2$. \citet{CFG-WFA} present a similar construction. Notation is the same as \autoref{fig:sigma-CM}.}
    \label{fig:binary}
\end{figure}
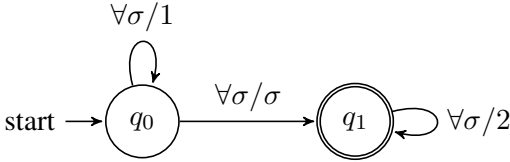

\subsection{Towards Transformers}

\autoref{sec:self-attention} presents preliminary results extending saturation analysis to self attention. We show saturated self attention is not RR and consider its space complexity. We hope further work will more completely characterize saturated self attention.
\section{Language Expressiveness} \label{sec:languages}
Having explored the set of functions expressible internally by different saturated RNN encoders, we
turn to the
languages recognizable when  using them 
with a decoder. 
We consider the following setup:

\begin{compactenum}
    \item An s-RNN encodes $x$ to a vector $\bv h_t \in \mathbb{Q}^k$.
    \item A \textit{decoder} function maps the last state $\bv h_t$ to an accept/reject decision, respectively: $\{1, 0\}$.
\end{compactenum}

We say that a language $L$ is decided by an encoder-decoder pair $\bv e, \bv d$ if $\bv d(\bv e(x))=1$  for every sequence $x \in L$ and otherwise $\bv d(\bv e(x))=0$. We explore which languages can be decided by different encoder-decoder pairings.

Some related results can be found in \citet{CFG-WFA}, who study the expressive power of WFAs in relation to CFGs under a slightly different definition of language recognition.

\subsection{Linear Decoders}
Let $\decfun{1}$ be the single-layer linear decoder
\begin{equation} \label{eq:decoder1}
    \decfun{1}(\bv h_t) \triangleq \mathbbm{1}_{>0}(\bv w \cdot \bv h_t + b) \in \{0, 1\}
\end{equation}
parameterized by $\bv w$ and $ b$.
For \mdet \model architecture $E$, we denote by $\MLPT{1}(E)$ the set of languages decidable by $E$ with $\decfun{1}$.
We use $\MLPT{2}(E)$ analogously 
for a 2-layer decoder with $\mathbbm{1}_{>0}$ activations, where the first layer has arbitrary width.

\subsection{A Decoder Adds Power}
We refer to sets of strings using regular expressions, e.g. $a^*=\{a^i \mid i\in \mathbb{N}\}$. To illustrate the purpose of the decoder, consider the following language:
\begin{equation} \label{eq:easy-lang}
    L_\leq = \{ x \in \{a, b\}^* \mid \num{a-b}{x} \leq 0 \} .
\end{equation}
The Hankel sub-block of the indicator function for $L_\leq$ over $P=a^*$, $S=b^*$ is lower triangular. Therefore, no RR encoder can compute it.

However, adding the $\MLPT{1}$ decoder allows us to compute this indicator function with an s-QRNN, which is RR. 
We set the s-QRNN layer to compute the simple series $c_t = \num{a-b}{x}$ (by increasing on $a$ and decreasing on $b$). 
The $\MLPT{1}$ layer then checks $c_t \leq 0$. 
So, while the indicator function for $L_\leq $ is not itself rational, it can be easily recovered from a rational representation. Thus, $L_\leq \in \MLPT{1}$(s-QRNN).

\subsection{Case Study: $\anbn$} \label{sec:anbn-summary}

We compare the language expressiveness of several rational and non-rational RNNs on the following:
\begin{align}
    a^nb^n &\triangleq \{ a^nb^n \mid n \in \mathbb{N} \} \\
    \contanbn &\triangleq \{a^nb^n(a|b)^* \mid 0 < n \} .
\end{align}
\noindent $a^nb^n$ is more interesting than $L_\leq$ because the $\MLPT{1}$ decoder cannot decide it simply by asking the \model to track $\num{a-b}{x}$, as that would require it to compute the non-linearly separable ${=}0$ function.
Thus, it appears at first that deciding $a^nb^n$ with $\MLPT{1}$ might require a non-rational RNN \model. 
However, we show below that this is not the case.

Let $\circ$ denote stacking two layers. We will go on to discuss the following results:
\begin{align}
    \anbn &\in \MLPT{1}(\clean{WFA}) \label{eq:anbn-wfa} \\
    \anbn &\in \MLPT{1}(\clean{s-LSTM}) \label{eq:anbn-lstm} \\
    \anbn &\not\in \MLPT{1}(\clean{s-QRNN}) \label{eq:anbn-qrnn} \\    
    \anbn &\in \MLPT{1}(\clean{s-QRNN $\circ$ s-QRNN}) \label{eq:anbn-qrnn2} \\
    \anbn &\in \MLPT{2}(\clean{s-QRNN}) \label{eq:anbn-qrnn+} \\
    \contanbn &\in \MLPT{1}(\clean{s-LSTM}) \label{eq:anbnw-lstm} \\
    \contanbn &\notin \MLPT{\ \!}(\clean{s-QRNN}) \clean{ for \emph{any} } \MLPT{\ \!} \label{eq:anbnw-qrnn*} \\
    \contanbn \cup \{ \epsilon \} &\in \MLPT{1}(\clean{s-QRNN $\circ$ s-QRNN}) \label{eq:anbnw-qrnn2}
\end{align}

\paragraph{WFAs {\normalfont (\autoref{sec:anbn})}} In \autoref{thm:anbnRR} we present a function $f:\Sigma^*\rightarrow\mathbb{Q}$ satisfying $f(x)>0$ iff $x\in\anbn$, and show that $H_f$ has finite rank. It follows that there exists a WFA that can decide $a^nb^n$ with the $\MLPT{1}$ decoder. Counterintuitively, $a^nb^n$ can be recognized using rational encoders.

\paragraph{QRNNs {\normalfont (\autoref{sec:sqrnn-anbn})}} Although $a^nb^n \in \MLPT{1}(\clean{WFA})$, it does not follow that every rationally recurrent model can also decide $\anbn$ with the help of $\MLPT{1}$.
Indeed, in \autoref{thm:anbn-qrnn}, we prove that 
$a^nb^n\notin \MLPT{1}(\clean{s-QRNN})$, whereas $a^nb^n \in \MLPT{1}(\clean{s-LSTM})$ (\autoref{thm:anbn-lstm}).


It is important to note that, with a more complex decoder, the QRNN \textit{could} recognize $\anbn$. 
For example, the s-QRNN can encode $c_1=\num{a-b}{x}$ and set $c_2$ to check whether $x$ contains $ba$, from which a $\MLPT{2}$ decoder can recognize $\anbn$ (\autoref{thm:qrnn2-anbn}).

This does not mean the hierarchy dissolves as the decoder is strengthened. 
We show that $\contanbn$---which seems like a trivial extension of $a^nb^n$---is not recognizable by the s-QRNN with \emph{any} decoder.

This result may appear counterintuitive, but in fact
highlights the s-QRNN's lack of counter awareness: it can only passively encode the information needed by the decoder to recognize $a^nb^n$. Failing to recognize that a valid prefix has been matched, it cannot act to preserve that information after additional input tokens are seen.
We present a proof in \autoref{thm:prefanbn-qrnn}. 
In contrast, in \autoref{thm:prefanbn-lstm} we show that the s-LSTM can directly encode an indicator for $\contanbn$ in its internal state.

\paragraph{Proof sketch:} $\contanbn\notin \MLPT{}(\clean{s-QRNN})$. A sequence $s_1\in \contanbn$ is shuffled to create $s_2\notin\contanbn$ with an identical 
multi-set of counter updates.\footnote{Since QRNN counter updates depend only on the $w$-grams present in the sequence.
}
Counter updates would be order agnostic if not for reset operations, and resets mask all history,
so extending $s_1$ and $s_2$ with a single suffix $s$ containing all of their $w$-grams
reaches the same final state. 
Then for any $D$, $\MLPT{}(\clean{s-QRNN})$ cannot separate them. We formalize this in \autoref{thm:prefanbn-qrnn}.

We refer to this technique as the \emph{suffix attack}, and note that
it can be used to prove for multiple other languages $L\in\MLPT{2}(\clean{s-QRNN})$
that $L\concat\Sigma^*$ is not in $\MLPT{}(\clean{s-QRNN})$ for any decoder $\MLPT{}$. 


\paragraph{2-layer QRNNs} Adding another layer overcomes the weakness of the 1-layer s-QRNN, at least for deciding $a^nb^n$. This follows from the fact that $a^nb^n \in \MLPT{2}(\clean{s-QRNN})$: the second QRNN layer can be used as a linear layer.

Similarly, we show in \autoref{thm:qrnn2-anbn} that a 2-layer s-QRNN can recognize $\contanbn \cup \{\epsilon\}$. This suggests that adding a second s-QRNN layer compensates for some of the weakness of the 1-layer s-QRNN, which, by the same argument for $\contanbn$ cannot recognize $\contanbn \cup \{\epsilon\}$ with any decoder.

\subsection{Arbitrary Decoder}

Finally, we study the theoretical case where the decoder is an arbitrary recursively enumerable (RE) function.
We view this as a loose upper bound of stacking many layers after a rational encoder.
What information is inherently lost by using a rational encoder? WFAs can uniquely encode each input, making them Turing-complete under this setup; however, this does not hold for rational s-RNNs.

\paragraph{RR-complete} Assuming an RR-complete encoder, a WFA like \autoref{fig:binary} can be used to encode each possible input sequence over $\Sigma$ to a unique number. 
We then use the decoder as an oracle to decide any RE language. 
Thus, an RR-complete encoder with an RE decoder is Turing-complete.

\paragraph{Bounded space} However, the $\Theta(\log n)$ space bound of saturated rational RNNs like the s-QRNN means these models cannot fully encode the input. 
In other words, some information about the prefix $x_{:t}$ must be lost in $\bv c_t$. 
Thus, rational s-RNNs are not Turing-complete with an RE decoder.
\section{Experiments}\label{Sec:Experiments}

\begin{figure}
    \centering
    \def\svgwidth{\columnwidth}
    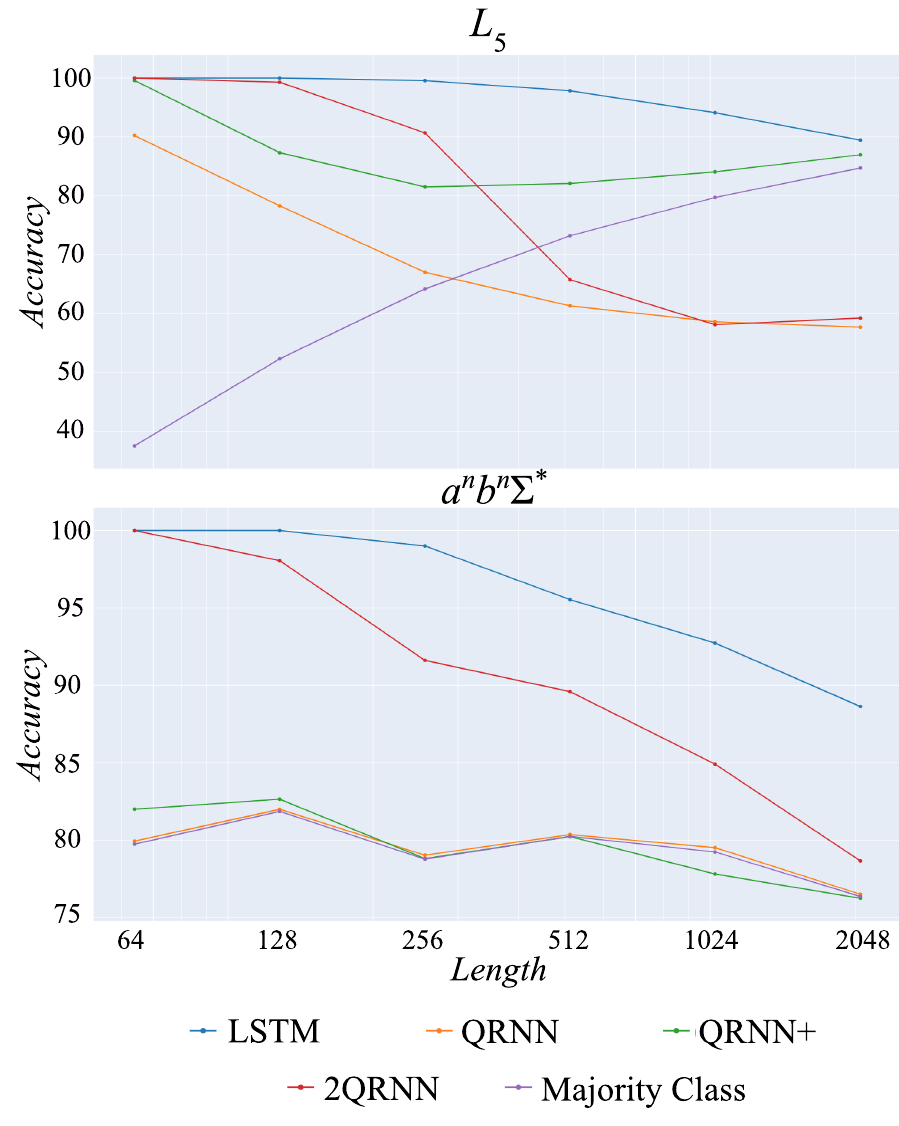
    \caption{Accuracy recognizing $L_5$
    and $\contanbn$.
    ``QRNN+" is a QRNN with a $2$-layer decoder, and ``2QRNN" is a $2$-layer QRNN with a $1$-layer decoder.
    \textbf{Experimental results updated in \autoref{sec:erratum}.}}
    \label{fig:generalization}
\end{figure}

In \autoref{sec:anbn-summary}, we showed that different saturated RNNs vary in their ability to recognize $\anbn$ and $\contanbn$.
We now test empirically whether these predictions carry over to the learnable capacity of unsaturated RNNs.\footnote{\url{https://github.com/viking-sudo-rm/rr-experiments}}
We compare the QRNN and LSTM
when coupled with a linear decoder $\MLPT{1}$.
We also train a $2$-layer QRNN (``QRNN2'') and a $1$-layer QRNN with a $\MLPT{2}$ decoder (``QRNN+'').

We train on strings of length $64$, and evaluate generalization on longer strings. We also compare to a baseline that always predicts the majority class.
The results are shown in \autoref{fig:generalization}.
We provide further experimental details in \autoref{sec:exp-details}.

\paragraph{Experiment 1} We use the following language, which has similar formal properties to $\anbn$, but with a more balanced label distribution:
\begin{equation} \label{eq:experiment}
    L_5 = \big\{ x \in (a|b)^* \; \mid \; \abs{\num{a-b}{x}} < 5 \big\} .
\end{equation}
In line with \eqref{eq:anbn-lstm}, the LSTM decides $L_5$ perfectly for $n\leq 64$, and generalizes fairly well to longer strings. As predicted in \eqref{eq:anbn-qrnn}, the QRNN cannot fully learn $L_5$ even for $n=64$. Finally, as predicted in \eqref{eq:anbn-qrnn2} and \eqref{eq:anbn-qrnn+}, the $2$-layer QRNN and the QRNN with $\MLPT{2}$ do learn $L_5$. However, we see that they do not generalize as well as the LSTM for longer strings. We hypothesize that these multi-layer models require more epochs to
reach
the same generalization performance as the LSTM.\footnote{As shown by the baseline, generalization is challenging because positive labels become less likely as strings get longer.}

\paragraph{Experiment 2} We also consider $\contanbn$. As predicted in \eqref{eq:anbnw-lstm} and \eqref{eq:anbnw-qrnn2}, the LSTM and $2$-layer QRNN decide $\contanbn$ flawlessly for $n=64$. A $1$-layer QRNN performs at the majority baseline for all $n$ with both a $1$ and $2$-layer decoder. Both of these failures were predicted in \eqref{eq:anbnw-qrnn*}. Thus, the only models that learned $\contanbn$ were exactly those predicted by the saturated theory.
\section{Conclusion} \label{Sec:Conclusion}
We develop a hierarchy of saturated RNN encoders, considering two angles: space complexity and rational recurrence.
Based on the hierarchy, we formally distinguish the state expressiveness of the non-rational s-LSTM and its rational counterpart, the s-QRNN.
We show further distinctions in state expressiveness based on encoder space complexity.

Moreover, the hierarchy translates to differences in language recognition capabilities.  
Strengthening the decoder alleviates some, but not all, of these differences.
We present two languages, both recognizable by an LSTM. 
We show that one can be recognized by an s-QRNN only with the help of a decoder,
and that the other cannot be recognized by an s-QRNN with the help of any decoder.

While this means existing rational RNNs are fundamentally limited compared to LSTMs, we find that it is not necessarily being rationally recurrent that limits them: in fact, we prove that a WFA can perfectly encode its input---something no saturated RNN can do.
We conclude with an analysis that shows that an RNN architecture's strength must also take into account its space complexity.
These results further our understanding of the inner working of NLP systems. We hope they will guide the development of more expressive rational RNNs.

\section*{Acknowledgments}
We appreciate Amir Yehudayoff's help in finding the WFA used in \autoref{thm:anbnRR}. We also thank our anonymous reviewers, Tobias Jaroslaw, Ana Marasovi\'{c}, and other researchers at the Allen Institute for AI.
The project was supported in part by NSF grant IIS-1562364, Israel Science Foundation grant no.1319/16, and the European Research Council under the EU's Horizon 2020 research and innovation program, grant agreement No. 802774 (iEXTRACT).

\bibliography{acl2020}
\bibliographystyle{acl_natbib}

\appendix
\section{Rational Counting} \label{sec:restricted-counting}

We extend the result in \autoref{thm:CM-level1} as follows.

\begin{theorem}
Any $(\Sigma \times Q)$-restricted CM is rationally recurrent.
\end{theorem}

\begin{proof}
We present an algorithm to construct a WFA computing an arbitrary counter in a $(\Sigma \times Q)$-restricted CM. First, we create two independent copies of the transition graph for the restricted CM. We refer to one copy of the CM graph as the \textit{add graph}, and the other as the \textit{multiply graph}.

The initial state in the add graph receives a starting weight of $1$, and every other state receives a starting weight of $0$. Each state in the add graph receives an accepting weight of $0$, and each state in the multiply graph receives an accepting weight of $1$. In the add graph, each transition receives a weight of $1$. In the multiply graph, each transition receives a weight of $0$ if it represents $\times0$, and $1$ otherwise. Finally, for each non-multiplicative update $\sigma/{+}m$\footnote{Note that $m=-1$ for the ${-}1$ counter update.} from $q_i$ to $q_j$ in the original CM, we add a WFA transition $\sigma/m$ from $q_i$ in the add graph to $q_j$ in the multiply graph.

Each counter update creates one path ending in the multiply graph. The path score is set to $0$ if that counter update is ``erased'' by a ${\times}0$ operation. Thus, the sum of all the path scores in the WFA equals the value of the counter.
\end{proof}

This construction can be extended to accommodate ${=}m$ counter updates from $q_i$ to $q_j$ by adding an additional transition from the initial state to $q_j$ in the multiplication graph with weight $m$. This allows us to apply it directly to s-QRNNs, whose update operations include ${=}1$ and ${=}{-}1$.
\section{WFAs} \label{sec:anbn}

We show that while WFAs cannot directly encode an indicator for the language $\anbn=\{a^nb^n|\mid n\in \mathbb{N}\}$, they can encode a function that can be thresholded to recognize $\anbn$, i.e.:

\begin{theorem}\label{thm:anbnRR}
The language $\anbn=\{a^nb^n \mid n\in\mathbb{N}\}$ over $\Sigma=\{a,b\}$ is in $\MLPT{1}(\mathrm{WFA})$.
\end{theorem}

We prove this by showing a function whose Hankel matrix has finite rank that, when combined with the identity transformation (i.e., $w=1,b=0$) followed by thresholding, is an indicator for $\anbn$. Using the shorthand $\sigma(x)=\#_\sigma(x)$, the function is:
\begin{equation}
f(w)=
    \begin{cases}
    0.5-2(a(x)-b(x))^2 & \textrm{if} \; x \in a^*b^* \\
    -0.5 & \textrm{otherwise} .
    \end{cases}
\end{equation}

Immediately $f$ satisfies $\mathbbm{1}_{>0}(f(x)) \iff x \in \anbn$. To prove that its Hankel matrix, $H_f$, has finite rank, we will create $3$ infinite matrices of ranks $3,3$ and $1$, which sum to $H_f$. The majority of the proof will focus on the rank of the rank $3$ matrices, which have similar compositions. 

We now show $3$ series $r,s,t$ and a set of series they can be combined to create. These series will be used to create the base vectors for the rank $3$ matrices.
\begin{align}
a_i&=\frac{i(i+1)}{2}\\
b_i&=i^2-1\\
r_i&=\fixfun_0(i,a_{i-2})\\
s_i&=\fixfun_1(i,-b_{i-1})\\
t_i&=\fixfun_2(i,a_{i-1})
\end{align}
\noindent where for every $j\leq 2$,
\begin{equation}
    \fixfun_j(i,x) =
    \begin{cases}
      x & \textrm{if} \; i>2 \\
      1 & \textrm{if} \; i=j \\
      0 & \textrm{otherwise.}
    \end{cases}
\end{equation}
    
\begin{lemma} Let $c_i = 1-2i^2$ and $\{c^{(k)}\}_{k\in\mathbb{N}}$ be the set of series defined $c^{(k)}_i=c_{|i-k|}$. Then for every $i,k\in\mathbb{N}$,
\begin{equation*}
    c^{(k)}_i=c^{(k)}_0r_i+c^{(k)}_1s_i+c^{(k)}_2t_i .
\end{equation*}
\end{lemma}

\begin{proof}
For $i \in \{0,1,2\}$, $r_i,s_i$ and $t_i$ collapse to a `select' operation, giving the true statement $c_i^{(k)}=c_i^{(k)}\cdot 1$.
We now consider the case $i>2$.
Substituting the series definitions in the right side of the equation gives

\begin{equation}
    c_{k}a_{i-2}+c_{|k-1|}(-b_{i-1})+c_{k-2}a_{i-1}
\end{equation}
\noindent which can be expanded to
\begin{align*}
&(1-2k^2) &&\cdot \; \frac{i^2-3i+2}{2} &&+ \\
&(1-2(k-1)^2) &&\cdot \; (1-(i-1)^2) &&+ \\
&(1-2(k-2)^2) &&\cdot \; \frac{(i-1)i}{2} .
\end{align*}

\noindent Reordering the first component and partially opening the other two gives
\begin{align*}
(-2k^2+1)\frac{i^2-3i+2}{2} &+\\
(-2k^2+4k-1)(2i-i^2) &+\\
(-k^2+4k-3.5)(i^2-i)
\end{align*}

\noindent and a further expansion gives
\begin{align*}
-k^2i^2&           +&&0.5i^2   +3k^2i          -1.5i  -2k^2 + 1 +\\
2k^2i^2&   -4ki^2     +&&i^2   -4k^2i  +8ki    -  2i            +\\
-k^2i^2&   +4ki^2  -&&3.5i^2   + k^2i  -4ki    +3.5i
\end{align*}
\noindent which reduces to
\begin{equation*}
    -2i^2+4ki-2k^2+1=1-2(k-i)^2=c^{(k)}_i .
\end{equation*}
\end{proof}

We restate this as:
\begin{corollary}\label{cor:base}
For every $k\in\mathbb{N}$, the series $c^{(k)}$ is a linear combination of the series $r,s$ and $t$.
\end{corollary}

We can now show that $f$ is computable by a WFA, proving \autoref{thm:anbnRR}. By \autoref{thm:rank}, it is sufficient to show that $H_f$ has finite rank.

\begin{lemma}
$H_f$ has finite rank.
\end{lemma}

\begin{proof}
For every $P,S\subseteq\{a,b\}^*$, denote 
$$[H_f|_{P,S}]_{u,v} =
    \begin{cases*}
      [H_f]_{u,v} & if $u\in P$ and $v\in S$ \\
      0 & otherwise
    \end{cases*}
$$
Using regular expressions to describe $P,S$, we create the 3 finite rank matrices which sum to $H_f$:
\begin{align}
    A&=(H_f+0.5)|_{a^*,a^*b^*} \\
    B&=(H_f+0.5)|_{a^*b^+,b^*} \\
    C&=(-0.5)|_{u,v} .
\end{align}
Intuitively, these may be seen as a ``split'' of $H_f$ into sections as in \autoref{fig:break_H}, such that $A$ and $B$ together cover the sections of $H_f$ on which $u\concat v$ does not contain the substring $ba$ (and are equal on them to $H_f+0.5$), and $C$ is simply the constant matrix $-0.5$. Immediately, $H_f = A+B+C$, and $\mathrm{rank}(C)=1$.

\begin{figure}
    \centering
    \includegraphics[scale=.25]{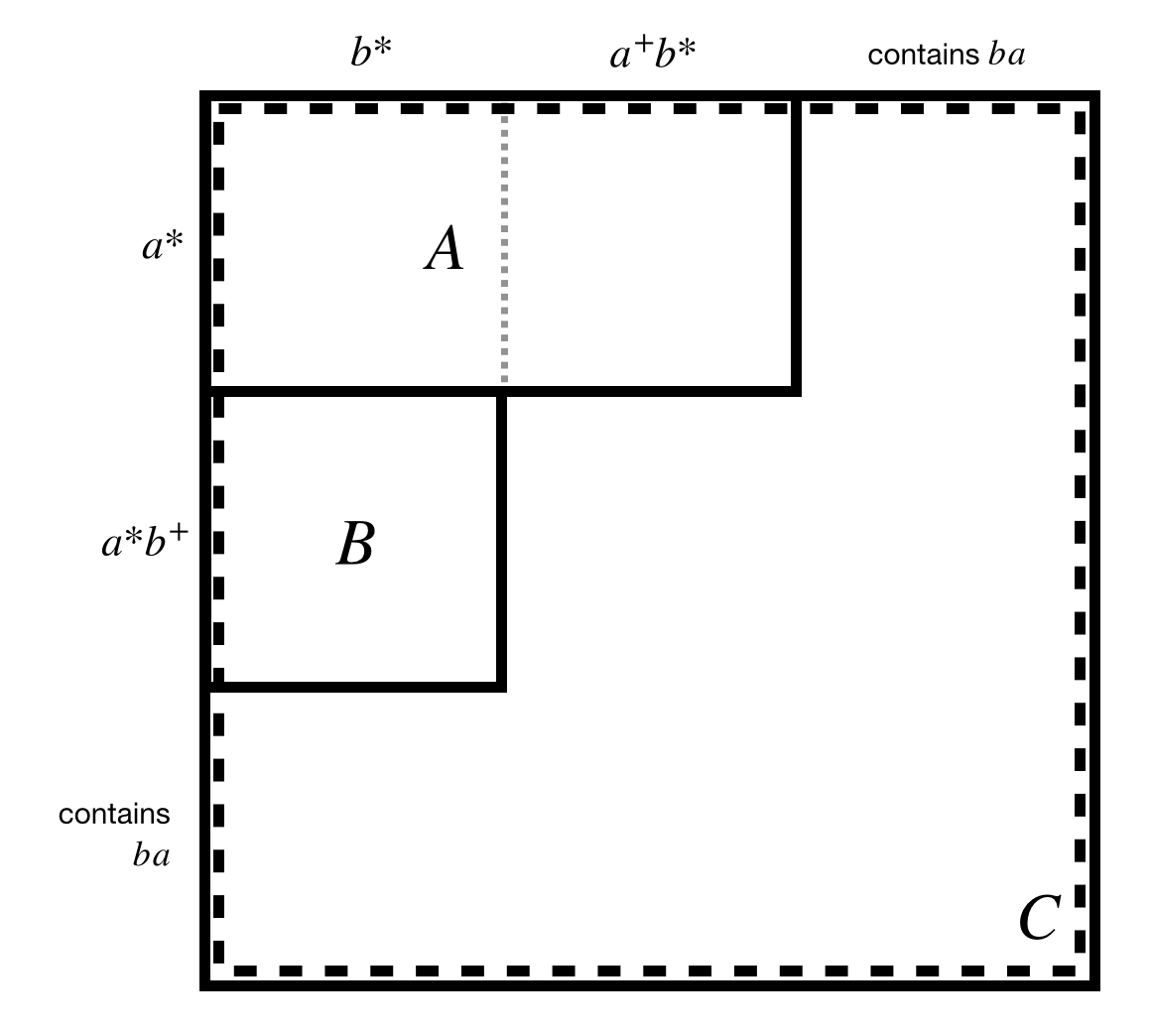}
    \caption{Intuition of the supports of $A,B$ and $C$.}
    \label{fig:break_H}
\end{figure}


We now consider $A$. Denote $P_A=a^*,S_A=a^*b^*$. $A$ is non-zero only on indices $u\in P_A,v\in S_A$, and for these, $u\concat v\in a^*b^*$ and $A_{u,v}=0.5+f(u\concat v)=1-2(a(u)+a(v)-b(v))^2$. This gives that for every $u\in P_A,v\in S_A$,
\begin{equation}
    A_{u,v}=c_{|a(u)-(b(v)-a(v))|}=c^{(a(u))}_{b(v)-a(v)} .
\end{equation}

\noindent For each $\tau \in \{r, s, t\}$, define $\tilde{\tau}\in\mathbb{Q}^{\{a,b\}^*}$ as
\begin{equation}
    \tilde{\tau}_v=\mathbbm{1}_{v\in a^*b^*}\cdot \tau_{b(v)-a(v)} .
\end{equation}

We get from \autoref{cor:base} that for every $u\in a^*$, the $u$th row of $A$ is a linear combination of $\tilde{r},\tilde{s},$ and $\tilde{t}$. The remaining rows of $A$ are all $\bv 0$ and so also a linear combination of these, and so $\mathrm{rank}(A)\leq 3$.

Similarly, we find that the nonzero entries of $B$ satisfy
\begin{equation}
    B_{u,v}=c_{|b(v)-(a(u)-b(u))|}= c^{(b(v))}_{a(u)-b(u)}
\end{equation}
and so, for $\tau \in \{r, s, t\}$, the columns of $B$ are linear combinations of the columns $\tau' \in \mathbb{Q}^{\{a,b\}^*}$ defined 
\begin{equation}
    \tau'_u=\mathbbm{1}_{u\in a^*b^+}\cdot \tau_{a(u)-b(u)} .
\end{equation}
\noindent Thus we conclude $\mathrm{rank}(B)\leq 3$.

Finally, $H_f=A+B+C$, and so by the subadditivity of rank in matrices,
\begin{equation}
    \mathrm{rank}(H_f)\leq \sum_{M=A,B,C}\mathrm{rank}(M)=7 .
\end{equation}
\end{proof}

In addition, the rank of $\tilde{H}_f\in \mathbb{Q}^{\{a,b\}^{\leq 2},\{a,b\}^{\leq 2}}$ defined $[\tilde{H}_f]_{u,v}=[H_f]_{u,v}$ is $7$, and so we can conclude that the bound in the proof is tight, i.e., $\mathrm{rank}(H_f)=7$. From here $\tilde{H}_f$ is a complete sub-block of $H_f$ and can be used to explicitly construct a WFA for $f$, using the spectral method described by \citet{Balle2014}.
\section{s-QRNNs} \label{sec:sqrnn-anbn}

\begin{theorem} \label{thm:anbn-qrnn}
No s-QRNN with a linear threshold decoder can recognize $\anbn=\{a^nb^n \mid n\in\mathbb{N}\}$, i.e., $\anbn\notin\MLPT{1}($s-QRNN$)$.
\end{theorem}

\begin{proof}
An \textit{ifo} s-QRNN can be expressed as a $\Sigma^k$-restricted CM with the additional update operations $\{\coloneqq -1, \coloneqq 1\}$, where $k$ is the window size of the QRNN. 
So it is sufficient to show that such a machine, when coupled with the decoder $\MLPT{1}$ (linear translation followed by thresholding), cannot recognize $\anbn$.

Let $\mathcal{A}$ be some such CM, with window size $k$ and $h$ counters. Take $n=k+10$ and for every $m\in\mathbb{N}$ denote $w_m=a^nb^m$ and the counter values of $\mathcal{A}$ after $w_m$ as  $c^m\in\mathbb{Q}^h$.
Denote by $u_t$ the vector of counter update operations made by this machine on input sequence $w_m$ at time $t\leq n+m$. 
As $\mathcal{A}$ is dependent only on the last $k$ counters, necessarily all $u_{k+i}$ are identical for every $i\geq 1$.

It follows that for all counters in the machine that go through an assignment (i.e., $\coloneqq$) operation in $u_{k+1}$, their values in $c^{k+i}$ are identical for every $i\geq 1$, and for every other counter $j$, $c^{k+i}_j-c^k_j=i\cdot \delta$ for some $\delta\in\mathbb{Z}$.
Formally: for every $i\geq 1$ there are two sets $I$, $J=[h]\setminus I$ and constant vectors $\bv u\in\mathbb{N}^{I}, \bv v\in\mathbb{N}^{J}$ s.t. $c^{k+i}|_{I}=\bv u$ and $[c^{k+i}-c^k]|_{J}=\bv i\cdot v$.

We now consider the linear thresholder, defined by weights and bias $\bv w, b$. In order to recognise $a^nb^n$, the thresholder must satisfy:
\begin{align}
    \bv w \cdot c^{k+9} + &b &&<0 \label{eq:inj} \\
    \bv w \cdot c^{k+10} + &b &&>0 \label{eq:jni}\\
    \bv w \cdot c^{k+11} + &b && <0\label{eq:n}
\end{align}

Opening these equations gives:

\begin{align}
	\bv w|_J(\cdot c^{k}|_J+& 9 \bv v|_J) &&+ \bv w|_I\cdot \bv u &&&< 0\\
		\bv w|_J(\cdot c^{k}|_J+& 10 \bv v|_J) &&+ \bv w|_I\cdot \bv u &&&> 0\\
	\bv w|_J(\cdot c^{k}|_J+& 11 \bv v|_J) &&+ \bv w|_I\cdot \bv u &&&< 0
\end{align}
\noindent but this gives
$9 w|_J{\cdot} \bv v|_J < 10  w|_J{\cdot} \bv v|_J > 11 w|_J{\cdot} \bv v|_J$, which is impossible.

\end{proof}

However, this does not mean that the s-QRNN is entirely incapable of recognising $a^nb^n$. Increasing the decoder power allows it to recognise $a^nb^n$ quite simply:

\begin{theorem} \label{thm:qrnn2-anbn}
For the two-layer decoder $\MLPT{2}$, $\anbn \in\MLPT{2}$(\clean{s-QRNN}).
\end{theorem}

\begin{proof}
Let $\num{ba}{x}$ denote the number of $ba$ $2$-grams in $x$. We use s-QRNN with window size $2$ to maintain two counters:
\begin{align}
    [\bv c_t]_1 &= \num{a-b}{x} \\
    [\bv c_t]_2 &= \num{ba}{x} .
\end{align}

\noindent $[\bv c_t]_2$ can be computed provided the QRNN window size is $\geq 2$. A two-layer decoder can then check
\begin{equation}
    0 \leq [\bv c_t]_1 \leq 0 \wedge [\bv c_t]_2 \leq 0 .
\end{equation}
\end{proof}

\begin{theorem}[Suffix attack] \label{thm:prefanbn-qrnn}
No s-QRNN and decoder can recognize the language $\contanbn=a^nb^n(a|b)^*$, $n>0$, i.e., $\contanbn\notin L ($s-QRNN$)$ for any decoder $L$.
\end{theorem}
The proof will rely on the s-QRNN's inability to ``freeze'' a computed value, protecting it from manipulation by future input.

\begin{proof}

As in the proof for \autoref{thm:anbn-qrnn}, it is sufficient to show that no $\Sigma^k$-restricted CM with the additional operations $\{{\coloneqq}{-}1,{\coloneqq}1 \}$ can recognize $\contanbn$ for any decoder $L$.

Let $\mathcal{A}$ be some such CM, with window size $k$ and $h$ counters. For every $w\in\Sigma^n$ denote by $c(w)\in\mathbb{Q}^h$ the counter values of $\mathcal{A}$ after processing $w$.
Denote by $u_t$ the vector of counter update operations made by this machine on an input sequence $w$ at time $t\leq |w|$. 
Recall that $\mathcal{A}$ is $\Sigma^k$ restricted, meaning that $u_i$ depends exactly on the window of the last $k$ tokens for every $i$.

We now denote $j=k+10$ and consider the sequences $w_1=a^jb^ja^jb^ja^jb^j$, $w_2=a^jb^{j-1}a^jb^{j+1}a^jb^j$. $w_2$ is obtained from $w_1$ by removing the $2j$-th token of $w_1$ and reinserting it at position $4j$.

As all of $w_1$ is composed of blocks of $\geq k$ identical tokens, the windows preceding all of the other tokens in $w_1$ are unaffected by the removal of the $2j$-th token. 
Similarly, being added onto the end of a substring $b^k$, its insertion does not affect the windows of the tokens after it, nor is its own window different from before. 
This means that overall, the set of all operations $u_i$ performed on the counters is identical in $w_1$ and in $w_2$. 
The only difference is in their ordering.

$w_1$ and $w_2$ begin with a shared prefix $a^k$, and so necessarily the counters are identical after processing it. We now consider the updates to the counters after these first $k$ tokens, these are determined by the windows of $k$ tokens preceding each update.

First, consider all the counters that undergo some assignment ($\coloneqq$) operation during these sequences, and denote by $\{w\}$ the multiset of windows in $w\in\Sigma^k$ for which they are reset. $w_1$ and $w_2$ only contain $k$-windows of types $a^xb^{k-x}$ or $b^xa^{k-x}$, and so these must all re-appear in the shared suffix $b^ja^jb^j$ of $w_1$ and $w_2$, at which point they will be synchronised. It follows that these counters all finish with identical value in $c(w_1)$ and $c(w_2)$.

All the other counters are only updated using addition of $-1,1$ and $0$, and so the order of the updates is inconsequential. It follows that they too are identical in $c(w_1)$ and $c(w_2)$, and therefore necessarily that $c(w_1)=c(w_2)$.

From this we have $w_1,w_2$ satisfying $w_1\in\contanbn,w_2\notin\contanbn$ but also $c(w_1)=c(w_2)$. Therefore, it is not possible to distinguish between $w_1$ and $w_2$ with the help of any decoder, despite the fact that $w_1\in \contanbn$ and $w_2\notin\contanbn$. It follows that the CM and s-QRNN cannot recognize $\contanbn$ with any decoder.
\end{proof}

For the opposite extension $\Sigma^* \anbn$, in which the language is augmented by a \emph{prefix}, we cannot use such a ``suffix attack''. In fact, $\Sigma^* \anbn$ can be recognized by an s-QRNN with window length $w\geq 2$ and a linear threshold decoder as follows: a counter counts $\#_{a-b}(x)$ and is reset to $1$ on appearances of $ba$, and the decoder compares it to $0$.

Note that we define decoders as functions from the final state to the output. Thus, adding an additional QRNN layer does not count as a ``decoder" (as it reads multiple states). In fact, we show that having two QRNN layers allows recognizing $\contanbn$.

\begin{theorem}
Let $\epsilon$ be the empty string. Then,
\begin{equation*}
    \contanbn \cup \{\epsilon\} \in \MLPT{1}(\clean{s-QRNN} \circ \clean{s-QRNN}) .
\end{equation*}
\end{theorem}

\begin{proof}
We construct a two-layer s-QRNN from which $\contanbn$ can be recognized. Let $\$$ denote the left edge of the string. The first layer computes two quantities $d_t$ and $e_t$ as follows:
\begin{align}
    d_t &= \num{ba}{x} \\
    e_t &= \num{\$b}{x} .
\end{align}
Note that $e_t$ can be interpreted as a binary value checking whether the first token was $b$.
The second layer computes $c_t$ as a function of $d_t, e_t,$ and $x_t$ (which can be passed through the first layer). We will demonstrate a construction for $c_t$ by creating linearly separable functions for the gate terms $f_t$ and $z_t$ that update $c_t$.
\begin{align}
    f_t &= \begin{cases}
        1 & \textrm{if} \; d_t \leq 0 \\
        0 & \textrm{otherwise}
    \end{cases} \\
    z_t &= \begin{cases}
        1 & \textrm{if} \; x_t = a \vee e_t \\
        -1 & \textrm{otherwise.}
    \end{cases}
\end{align}
Now, the update function $u_t$ to $c_t$ can be expressed
\begin{equation}
    u_t = f_tz_t = \begin{cases}
        {+}0 & \textrm{if} \; 0 < d_t \\
        {+}1 & \textrm{if} \; d_t \leq 0 \wedge ( x_t = a \vee e_t ) \\
        {-}1 & \textrm{otherwise.}
    \end{cases}
\end{equation}
Finally, the decoder accepts iff $c_t \leq 0$. To justify this, we consider two cases: either $x$ starts with $b$ or $a$. If $x$ starts with $b$, then $e_t = 0$, so we increment $c_t$ by $1$ and never decrement it. Since $0 < c_t$ for any $t$, we will reject $x$. If $x$ starts with $a$, then we accept iff there exists a sequence of $b$s following the prefix of $a$s such that both sequences have the same length.
\end{proof}

\section{s-LSTMs}\label{sec:lstm-anbn}

In contrast to the s-QRNN, we show that the s-LSTM paired with a simple linear and thresholding decoder can recognize both $\anbn$ and $\contanbn$.

\begin{theorem} \label{thm:anbn-lstm}
\begin{equation*}
    \anbn \in \MLPT{1}(\clean{s-LSTM}) .
\end{equation*}
\end{theorem}

\begin{proof}
Assuming a string $a^ib^i$, we set two units of the LSTM state to compute the following functions using the CM in \autoref{fig:slstm-cm}:
\begin{align}
    [\bv c_t]_1 &= \relu(i-j) \\
    [\bv c_t]_2 &= \relu(j-i) .
\end{align}
We also add a third unit $[\bv c_t]_3$ that tracks whether the $2$-gram $ba$ has been encountered, which is equivalent to verifying that the string has the form $a^ib^i$. Allowing $\bv h_t = \tanh(\bv c_t)$, we set the linear threshold layer to check
\begin{equation}
    [\bv h_t]_1 + [\bv h_t]_2 + [\bv h_t]_3 \leq 0 .
\end{equation}
\end{proof}

\begin{theorem} \label{thm:prefanbn-lstm}
\begin{equation*}
    \contanbn \in \MLPT{1}(\clean{s-LSTM}) .
\end{equation*}
\end{theorem}

\begin{proof}
We use the same construction as \autoref{thm:anbn-lstm}, augmenting it with
\begin{equation}
    [\bv c_t]_4 \triangleq [\bv h_{t-1}]_1 + [\bv h_{t-1}]_2 + [\bv h_{t-1}]_3 \leq 0 .
\end{equation}
We decide $x$ according to the (still linearly separable) equation
\begin{equation}
    \big( 0 < [\bv h_t]_4 \big) \vee \big( [\bv h_t]_1 + [\bv h_t]_2 + [\bv h_t]_3 \leq 0 \big) .
\end{equation}
\end{proof}
\section{Experimental Details} \label{sec:exp-details}

Models were trained on strings up to length $64$, and, at each index $t$, were asked to classify whether or not the prefix up to $t$ was a valid string in the language. Models were then tested on independent datasets of lengths $64, 128, 256, 512, 1024,$ and $2048$. The training dataset contained $100000$ strings, and the validation and test datasets contained $10000$. We discuss task-specific schemes for sampling strings in the next paragraph. All models were trained for a maximum of $100$ epochs, with early stopping after $10$ epochs based on the validation cross entropy loss. We used default hyperparameters provided by the open-source AllenNLP framework \citep{Gardner2017AllenNLP}. The code is available at \url{https://github.com/viking-sudo-rm/rr-experiments}.

\paragraph{Sampling strings} For the language $L_5$, each token was sampled uniformly at random from $\Sigma = \{a, b\}$. For $\contanbn$, half the strings were sampled in this way, and for the other half, we sampled $n$ uniformly between $0$ and $32$, fixing the first $2n$ characters of the string to $a^nb^n$ and sampling the suffix uniformly at random.

\paragraph{Experimental cost} The originally reported experiments were run for 20 GPU hours on Quadro RTX 8000.
\section{Self Attention} \label{sec:self-attention}

\paragraph{Architecture} We place saturated self attention \citep{vaswani2017attention} into the state expressiveness hierarchy. We consider a single-head self attention encoder that is computed as follows:
\begin{enumerate}
    \item At time $t$, compute queries $\bv q_t$, keys $\bv k_t$, and values $\bv v_t$ from the input embedding $\bv x_t$ using a linear transformation.
    \item Compute attention head $\bv h_t$ by attending over the keys and values up to time $t$ ($\bv K_{:t}$ and $\bv V_{:t}$) with query $\bv q_t$.
    \item Let $\norm{\cdot}_L$ denote a layer normalization operation \citep{ba2016layer}.
        \begin{align}
            \bv h'_t &= \relu \big(\bv W^h \cdot \norm{\bv h_t}_L \big) \label{eq:layer1} \\
            \bv c_t &= \norm{\bv W^c \bv h'_t}_L .
        \end{align}
\end{enumerate}

This simplified architecture has only one attention head, and does not incorporate residual connections. It is also masked (i.e., at time $t$, can only see the prefix $\bv X_{:t}$), which enables direct comparison with unidirectional RNNs. For simplicity, we do not add positional information to the input embeddings.

\begin{theorem} \label{thm:transformer-rr}
Saturated masked self attention is not RR.
\end{theorem}

\begin{proof}
Let $\num{\sigma}{x}$ denote the number of occurences of $\sigma \in \Sigma$ in string $x$. We construct a self attention layer to compute the following function over $\{a, b\}^*$:
\begin{equation}
    f(x) = \begin{cases}
        0 & \textrm{if} \; \num{a}{x} = \num{b}{x} \\
        1 & \textrm{otherwise} .
    \end{cases}
\end{equation}
\noindent Since the Hankel sub-block over $P=a^*,S=b^*$ has infinite rank, $f \not\in \rat$.

Fix $\bv v_t = \bv x_t$. As shown by \citet{merrill-2019-sequential}, saturated attention over a prefix of input vectors $\bv X_{:t}$ reduces to sum of the subsequence for which key-query similarity is maximized, i.e., denoting $I=\{i\in [t]\mid \bv k_i\cdot \bv q_t = m\}$ where $m=\max\{\bv k_i\cdot \bv q_t|i\in [t]\}$:
\begin{equation}
    \bv h_t = \frac{1}{\abs{I}} \sum_{i\in I} \bv x_{t_i} .
\end{equation}
\noindent For all $t$, set the key and query $k_t, q_t = 1$. 
Thus, all the key-query similarities are $1$, and we obtain:
\begin{align}
    \bv h_t &= \frac{1}{t} \sum_{t'=1}^t \bv x_{t'} \\
    &= \frac{1}{t} \; \big( \num{a}{x}, \; \num{b}{x} \big)^\top .
\end{align}
Applying layer norm to this quantity preserves equality of the first and second elements. Thus, we set the layer in \eqref{eq:layer1} to independently check $0 < [\bv h^0_t]_1 - [\bv h^0_t]_2$ and $[\bv h^0_t]_1 - [\bv h^0_t]_2 < 0$ using $\relu$. The final layer $c_t$ sums these two quantities, returning $0$ if neither condition is met, and $1$ otherwise.

Since saturated self attention can represent $f \notin \rat$, it is not RR.
\end{proof}

\paragraph{Space Complexity}  We show that self attention falls into the same space complexity class as the LSTM and QRNN. 
Our method here extends \citet{merrill-2019-sequential}'s analysis of attention. 

\begin{theorem}
Saturated single-layer self attention has $\Theta(\log n)$ space.
\end{theorem}

\begin{proof}
The construction from \autoref{thm:transformer-rr} can reach a linear (in sequence length) number of different outputs, implying a linear number of different configurations,
and so that the space complexity of saturated self attention is $\Omega(\log n)$. We now show the upper bound $O(\log n)$.

A sufficient representation for the internal state (configuration) of a self-attention layer is the unordered group of key-value pairs over the prefixes of the input sequence.

Since $f_k:x_t\mapsto \bv k_t$ and $f_v: x_t\mapsto \bv v_t$ have finite domain ($\Sigma$), their images $K=\mathrm{image}(f_k),V=\mathrm{image}(f_v)$ are finite.\footnote{Note that any periodic positional encoding will also have finite image.}
Thus, there is also a finite number of possible key-value pairs $\langle \bv k_t, \bv v_t \rangle \in K \times V$. 
Recall that the internal configuration can be specified by the number of occurrences of each possible key-value pair.
Taking $n$ as an upper bound for each of these counts, we bound the number of configurations of the layer as $n^{\abs{K \times V}}$. Therefore the bit complexity is
\begin{equation}
    \log_2 \big( n^{\abs{K \times V}} \big) = O(\log n) .
\end{equation}
\end{proof}

Note that this construction does not apply if the ``vocabulary" we are attending over is not finite. Thus, using unbounded positional embeddings, stacking multiple self attention layers, or applying attention over other encodings with unbounded state might reach $\Theta(n)$.

While it eludes our current focus, we hope future work will extend the saturated analysis to self attention more completely. We direct the reader to \citet{hahn2020theoretical} for some additional related work.
\section{Memory Networks} \label{sec:stack-rnns}

All of the standard RNN architectures considered in \autoref{Sec:Hierarchy} have $O(\log n)$ space in their saturated form. In this section, we consider a stack RNN encoder similar to the one proposed by \citet{suzgun2019memory} and show how it, like a WFA, can encode binary representations from strings. Thus, the stack RNN has $\Theta(n)$ space. Additionally, we find that it is not RR. This places it in the upper-right box of \autoref{fig:hierarchy1}.

Classically, a stack is a dynamic list of objects to which elements $v \in V$ can be added and removed in a LIFO manner (using \emph{push} and \emph{pop} operations). 
The stack RNN proposed in \citet{suzgun2019memory} maintains a differentiable variant of such a stack, as follows:

\paragraph{Differentiable Stack}
In a differentiable stack, the update operation takes an element $s_t$ to push and a distribution $\pi_t$ over the update operations push, pop, and no-op, and returns the weighted average of the result of applying each to the current stack. 
The averaging is done elementwise along the stacks, beginning from the top entry.
To facilitate this, differentiable stacks are padded with infinite `null entries'. Their elements must also have a weighted average operation defined.



\begin{definition}[Geometric $k$-stack RNN encoder]
Initialize the stack $\bv S$ to an infinite list of null entries, and denote by $S_t$ the stack value at time $t$. Using $1$-indexing for the stack and denoting $\bv [S_{t-1}]_0 \triangleq \bv s_t$, the geometric $k$-stack RNN recurrent update is:\footnote{Intuitively, $[\pi_t]_a$ corresponds to the operations push, no-op, and pop, for the values $a=1,2,3$ respectively.}
\begin{align*}
    \bv s_t &= \bv f_s(x_t, \bv c_{t-1}) \\
    \pi_t  &= \bv f_\pi(x_t, \bv c_{t-1})\\
    \forall i \geq 1 \quad [\bv S_t]_i &= \sum_{a=1}^3 [\pi_t]_a [\bv S_{t-1}]_{i + a - 2}.
\end{align*}
\end{definition}

In this work we will consider the case where the null entries are $\bv 0$ and the encoding $\bv c_t$ is produced as a geometric-weighted sum of the stack contents, $$ \bv c_t = \sum_{i=1}^\infty \big( \frac{1}{2} \big)^{i-1} [\bv S_t]_i. $$ This encoding gives preference to the latest values in the stack, giving initial stack encoding $\bv c_0=\bv 0$.

\paragraph{Space Complexity} The memory introduced by the stack data structure pushes the encoder into $\Theta(n)$ space. We formalize this by showing that, like a WFA, the stack RNN can encode binary strings to their value.

\begin{lemma}
    The saturated stack RNN can compute the converging binary encoding function, i.e., $101 \mapsto 1 \cdot 1 + 0.5 \cdot 0 + 0.25 \cdot 1 = 1.25$.
\end{lemma}

\begin{proof}
Choose $k=1$. Fix the controller to always push $x_t$. Then, the encoding at time $t$ will be
\begin{equation}
    \bv c_t = \sum_{i=1}^t \big( \frac{1}{2} \big)^{i-1} x_i .
\end{equation}
This is the value of the prefix $x_{:t}$ in binary.
\end{proof}

\paragraph{Rational Recurrence} We provide another construction to show that the stack RNN can compute non-rational series. Thus, it is not RR.

\begin{definition}[Geometric counting]
    Define $f_2 : \{a, b\}^* \rightarrow \mathbb{N}$ such that
    \begin{equation*}
        f_2(x) = \exp_{\frac{1}{2}} \big( \num{a-b}{x} \big) - 1 .
    \end{equation*}
\end{definition}

\noindent Like similar functions we analyzed in \autoref{Sec:Hierarchy}, the Hankel matrix $H_{f_2}$ has infinite rank over the sub-block $a^ib^j$.

\begin{lemma}
    The saturated stack RNN can compute $f_2$.
\end{lemma}

\begin{proof}
    Choose $k=1$. Fix the controller to push $1$ for $x_t=a$, and pop otherwise.
\end{proof}
\begin{figure}
    \centering
    \includegraphics[width=.48\textwidth]{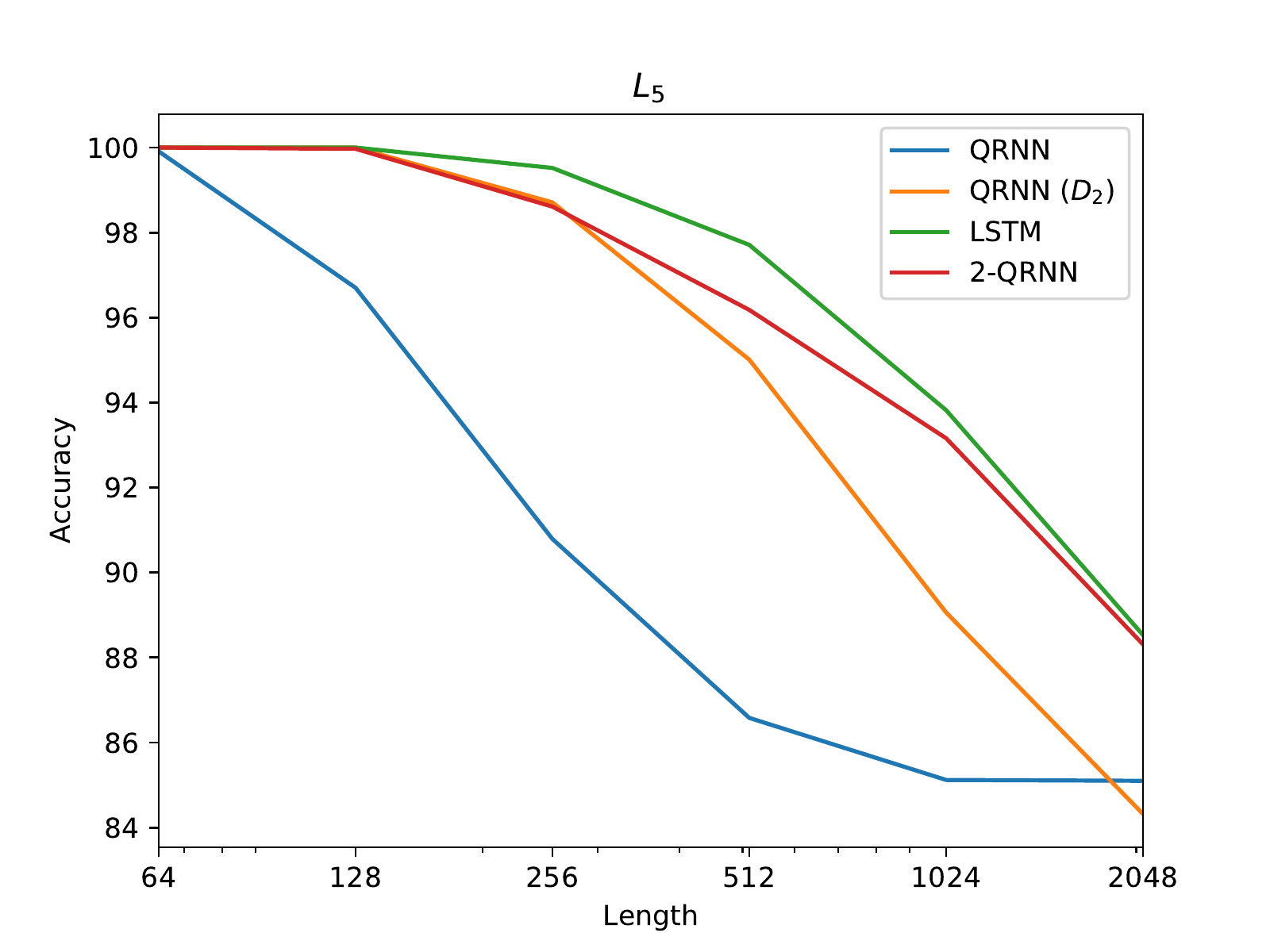}
    \includegraphics[width=.48\textwidth]{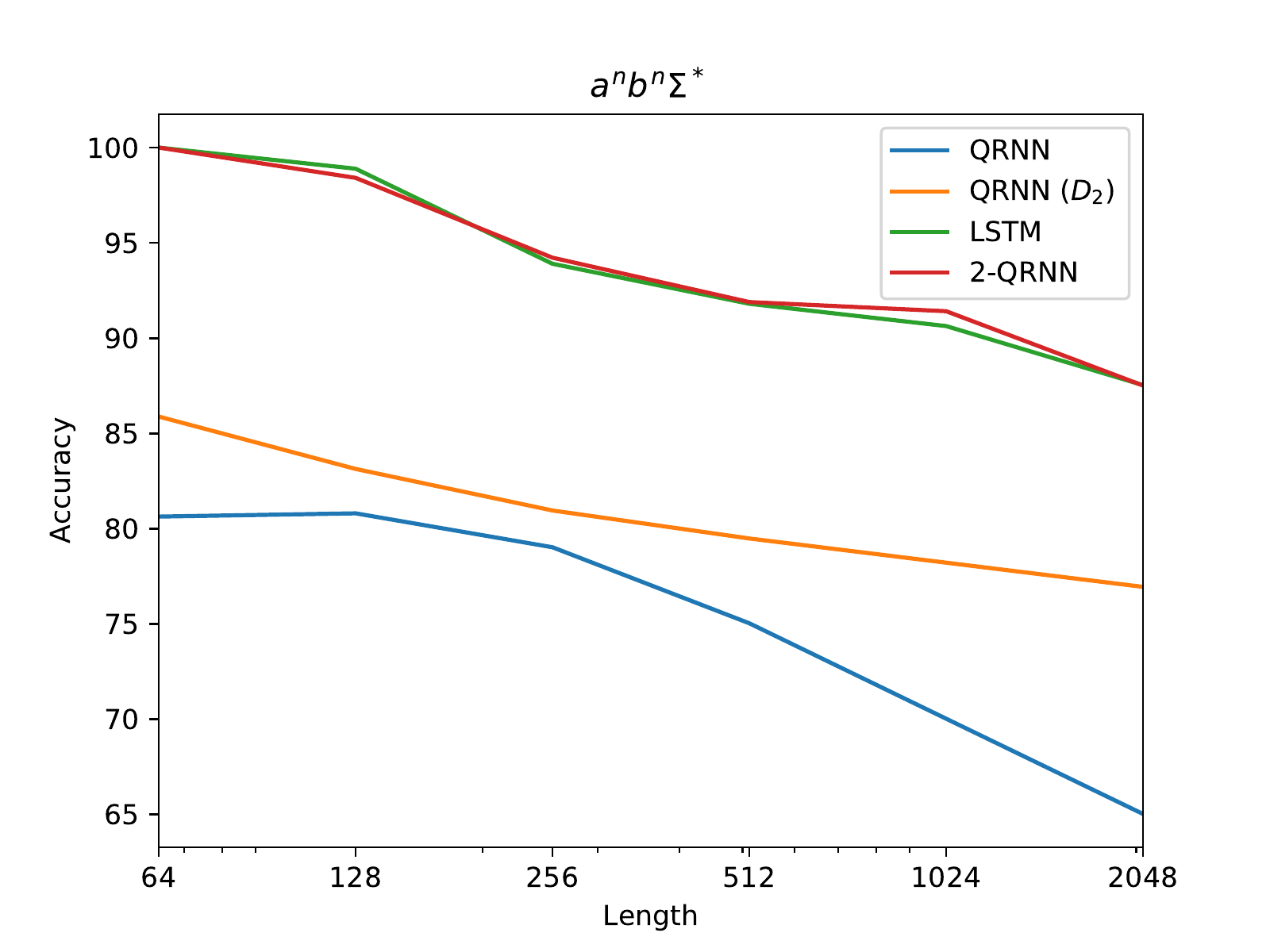}
    \caption{Updated results for $L_5$ (\textbf{top}) and $a^nb^n\Sigma^*$ (\textbf{bottom}). All networks use a $D_1$ decoder, except for ``QRNN ($D_2$)''. 2-QRNN is a 2-layer QRNN.}
    \label{fig:corrected}
\end{figure}

\section{Erratum} \label{sec:erratum}

We present corrections for experimental results originally reported in \autoref{Sec:Experiments}. Thanks to David Chiang for helping to identify these mistakes. The QRNN used in our theoretical analysis was the \textit{ifo}-QRNN, whereas QRNN used for the original experimental results was the \textit{fo}-QRNN (by our definition, a QRNN where $\mathbf{i} = 1 - \mathbf{f}$). We redo our experiments as originally intended with the \textit{ifo}-QRNN instead of the weaker \textit{fo}-QRNN.

Results are presented in \autoref{fig:corrected}. Overall, the trend is similar to what was originally reported. For $L_5$, all four models achieve 100\% accuracy at the training length of $64$. However, the QRNN performance drops earlier than for the other networks. This matches the theoretical result that the s-QRNN cannot recognize $L_5$, whereas the other three saturated networks can. For $a^nb^n\Sigma^*$, the LSTM and $2$-layer QRNN reach similar accuracy at all lengths. On the other hand, the $1$-layer QRNN, with either a $1$ or $2$-layer decoder, performs worse. This is predicted by the fact that the s-QRNN cannot recognize $a^nb^n\Sigma^*$ for any decoder.

While the results are mostly similar to the original results, one difference is that the \textit{ifo}-QRNN reaches 100\% accuracy on $L_5$ whereas the original QRNN did not reach 100\% even at $n=64$. We consider the \textit{generalization} accuracy for $n > 64$ to be a better indicator of whether the network has learned the language rather than the in-distribution test accuracy on strings of length $64$. This is because, if we evaluate at the same length, a finite-state model can still in principle do well since it is unlikely that the test set will contain prefixes with configurations unseen during training.

We formalize this for $L_5$, defined as
\begin{equation}
    L_5 = \big\{ x \in (a|b)^* \; \mid \; \abs{\num{a-b}{x}} < 5 \big\} .
\end{equation}

Define the \textit{configuration} $c(x)$ of a string $x \in \{a,b\}^*$ as $\num{a}{x} - \num{b}{x}$. Intuitively, $c(x)$ represents all the information needed solve the recognition task. As a function of string length $n$, $c(x)$ follows a random walk where the motion of each discrete time step is $1$ with probability $1/2$ and $-1$ otherwise. Thus, $c(x)$ is a random variable with a binomial distribution with mean $0$ and variance $n/4$. So, roughly 95\% of strings with length 64 will have $\abs{c(x)} \leq \sqrt{64} = 8$. Only by increasing the length $n$ can we force the model to contend with new configurations.


\end{document}